\documentclass{article}

\PassOptionsToPackage{numbers, compress}{natbib}

\usepackage[preprint]{neurips_2026}

\makeatletter
\renewcommand{\@toptitlebar}{%
  \vspace{-20pt}
  \par\noindent
  \begin{minipage}{\textwidth}
    \includegraphics[height=1.0cm]{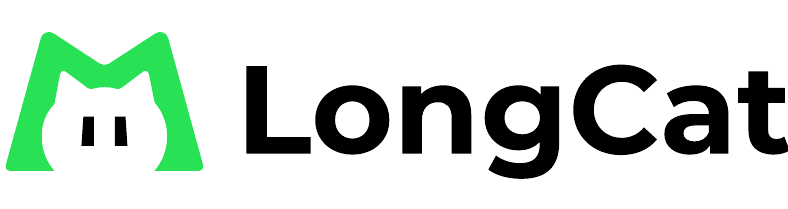}
  \end{minipage}
  \par\vspace{5pt}
  \hrule height 1pt
  \vskip .25in
}

 \usepackage[preprint]{neurips_2026}

\usepackage[colorlinks=true, citecolor=blue]{hyperref}
\usepackage[utf8]{inputenc} 
\usepackage[T1]{fontenc}    
\usepackage{hyperref}       
\usepackage{url}            
\usepackage{booktabs}       
\usepackage{amsfonts}       
\usepackage{nicefrac}       
\usepackage{microtype}      
\usepackage{xcolor}         
\usepackage{subfiles}
\usepackage{graphicx}
\usepackage{makecell}
\usepackage{amsmath}
\usepackage{pifont}
\usepackage{multirow}
\usepackage{caption}
\usepackage{enumitem}
\usepackage{algorithmic}
\usepackage{wrapfig}
\usepackage{multicol}
\usepackage[linesnumbered,ruled,vlined]{algorithm2e}
\usepackage{cleveref}
\usepackage{colortbl}

\newtheorem{theorem}{Theorem}
\newtheorem{lemma}{Lemma}
\newtheorem{proof}{Proof}

\newtheorem{corollary}{Corollary} 

\title{DORA: A Scalable Asynchronous Reinforcement
Learning System for Language Model Training}

\makeatletter
\renewcommand{\thefootnote}{\fnsymbol{footnote}}
\makeatother

\author{%
Tianhao Hu\footnotemark[2], Xiangcheng Liu\footnotemark[2], Yuchun Miao\footnotemark[2], Youshao Xiao\footnotemark[2], Hongyu Zang, Yang Zheng\\ \textbf{Xuan Huang}, \textbf{Jinrui Ding}, \textbf{Yufei Zhang}, \textbf{Yu Yang}, \textbf{Yi-Kai Zhang}, \textbf{Yueqing Sun}\\ \textbf{Chengcheng Han}, \textbf{Xiandi Ma}\footnotemark[1], \textbf{Wei Wang}, \textbf{Qi Gu}\footnotemark[1], \textbf{Yerui Sun}, \textbf{Yuchen Xie}, \textbf{Xunliang Cai}\\
Meituan Longcat Team\\
\texttt{maxiandi02@meituan.com}, \texttt{guqi03@meituan.com}
}

\begin{document}

\maketitle

\footnotetext[2]{Equal contribution.}
\footnotetext[1]{Corresponding authors.}

\makeatletter
\renewcommand{\thefootnote}{\arabic{footnote}}
\setcounter{footnote}{0}
\makeatother

\begin{abstract}
Reinforcement learning (RL) has become a critical paradigm 
for large-scale post-training \textcolor{black}{of LLMs} in industrial settings,
yet it faces a structural \emph{long-tail dilemma}:
rollout efficiency is bottlenecked by \textcolor{black}{the longest trajectories, which are often the most valuable for RL training.} 
Existing approaches alleviate this dilemma at the cost of either system overhead (e.g., re-prefill in partial-rollout methods) or algorithmic compromises (e.g., discarded long trajectories in replication-based methods). 
Both approaches implicitly assume that all rollout instances must synchronize around a single policy version, creating an unavoidable batch barrier.
We propose \textbf{DORA} (\textbf{D}ynamic \textbf{OR}chestration for \textbf{A}synchronous Rollout), 
which breaks this assumption by maintaining multiple policy versions concurrently within the rollout cluster to solve the skewed generation problem without \textcolor{black}{algorithmic} compromise. 
DORA combines three mechanisms: \emph{multi-version streaming training} that decouples trajectory completion from batch barrier, a centralized \emph{load-balancing orchestrator} that re-partitions resources across versions, and nearly \emph{zero-re-prefill migration} that transfers KV-Cache directly across same-version instances. Experiments on open-source benchmarks show that DORA achieves up to $2.12\times$ end-to-end throughput improvement and $8.2\times$ rollout-stage acceleration over synchronous training while preserving convergence parity. 
\textbf{In real-world production deployment} with thousands of accelerators, 
DORA achieves up to $6.2\times$ rollout speedup,
successfully training
a ${\sim}$500B-parameter MoE model (i.e., \textbf{LongCat-Flash-Thinking}) competitive with 
state-of-the-art open-source LLMs.

\end{abstract}

\section{Introduction}\label{sec:intro}
Reinforcement Learning (RL) has become a pivotal paradigm for LLM post-training, leveraging test-time scaling~\citep{snell2024scaling} to advance complex reasoning and agentic capabilities~\citep{Claude-Opus-4.5,openai_o1,deepseek-math-v2,team2025introducing,team2026longcat}. The RL training loop sequentially cycles through rollout (trajectory generation), experience preparation (reward and reference computation), and model training. Among these \textcolor{black}{stages}, rollout accounts for 50\%--80\% of the total step duration and represents the fundamental training bottleneck~\citep{team2025longcat,wu2025llamarl,xiao2023adaptive}. As industrial deployments scale to thousands of accelerators, optimizing rollout efficiency directly determines the overall training cost and turnaround time.


This rollout bottleneck is fundamentally driven by a \textbf{\textit{long-tail dilemma}} inherent to complex reasoning tasks: a direct conflict between algorithmic value and hardware efficiency. In mathematics and coding domains, response lengths follow a highly skewed distribution where the $99$th-percentile output can exceed the median by over an order of magnitude (Figures~\ref{fig:resp_in_house} and~\ref{fig:resp_in_house_2}). The dilemma arises because these exceptionally long trajectories carry the highest information density---encapsulating the intricate chain-of-thought (CoT) reasoning steps that are the primary source of emergent capabilities~\citep{openai_o1,guo2025deepseek}---making them indispensable for RL training. However, because decoding is inherently memory-bound, we cannot simply accelerate them by scaling compute, nor can we discard them without catastrophic algorithmic degradation. Consequently, under standard synchronous training, the entire batch is held hostage by a small fraction of these most valuable, yet longest trajectories, leaving the majority of devices completely idle.


To address this dilemma, asynchronous training has been proposed, with current efforts primarily focusing on two directions. \textit{Replication-based methods}~\citep{gao2025rollpacker,zhang2025sortedrl} shorten rollout duration by oversampling and dropping the in-flight long trajectories once enough complete. This \textit{discards} precisely the chain-of-thought trajectories highlighted above, and the resulting \textit{length-biased} distribution distorts the advantage estimation in group-relative methods such as GRPO~\citep{shao2024deepseekmath}. \textit{Partial-rollout methods}~\citep{team2025kimi,wu2025llamarl,fu2025areal,slime_github} segment long trajectories at each weight update and resume them under the new policy. On the system side, every update invalidates the KV-Cache and forces a \textit{full re-prefill} that grows dramatically with context length (Figure~\ref{fig:prompt_in_house}) and is further amplified in MoE architectures. On the algorithmic side, a trajectory now stitches multiple policy versions, potentially downgrades the rollout quality, and departs from the standard RL formulation, \textcolor{black}{risking model performance degradation.} Across both directions, existing approaches \textbf{alleviate the long-tail dilemma, but at the cost of either additional system overhead or algorithmic compromises}.

We argue that these tradeoffs stem from a common implicit assumption---\textbf{\textit{single-version rollout}}, i.e., the rollout instances only serve a single policy version. Under this assumption, in-flight long-tailed trajectories must either complete before the next policy update or be sacrificed at the update---leaving no room for solutions that avoid both system overhead and algorithmic compromises.

In this paper, we \textcolor{black}{depart from} this regime through \textbf{DORA} (\textbf{D}ynamic \textbf{OR}chestration for \textbf{A}synchronous Rollout), which embodies a \textbf{\textit{multi-version rollout}} paradigm where multiple policy versions coexist within the rollout instances, \textbf{resolving the long-tail dilemma without incurring significant system overhead or algorithmic compromises}. At its core, each trajectory is generated entirely under the policy version active at its dispatch, so that long-tailed trajectories run to completion under their original version while new requests proceed under the latest version.

Realizing this design, however, raises three system-level challenges, each addressed by a corresponding mechanism. \textbf{Version coexistence.} Maintaining multiple policy versions concurrently requires breaking the synchronous batch barrier so that completed trajectories can flow into training without waiting for the slowest ones. DORA achieves this through \textit{multi-version streaming training}, which dispatches and collects trajectories at the granularity of individual requests across versions, with a sliding window that bounds policy staleness. \textbf{Resource fragmentation.} As trajectories of older versions complete, their rollout instances become progressively underutilized while the latest version is over-subscribed. DORA addresses this through a centralized \textit{load-balancing orchestrator}, which continuously re-partitions data-parallel groups across versions in proportion to their pending workloads and migrates requests to rebalance the cluster. \textbf{Migration overhead.} Naively migrating a request across instances would re-trigger the prefill phase, \textcolor{black}{which is costly in long-context and MoE settings.} DORA avoids this entirely: since the trajectory is generated under the consistent policy version across multiple RL steps, its KV-Cache states are mathematically equivalent across any instance hosting that version, enabling nearly \textit{zero-re-prefill migration} via direct cross-instance KV-Cache transfer. Together, these mechanisms enable DORA to deliver \textbf{substantial efficiency gains while maintaining standard RL convergence.} Our main contributions are summarized as follows:

\noindent$\bullet$ \textbf{Multi-Version Streaming Training.} We propose \textit{multi-version streaming training}, a new asynchronous paradigm that maintains multiple policy versions concurrently within the rollout instance, eliminating the long-tail bubble without introducing significant system overhead or algorithmic compromises.

\noindent$\bullet$ \textbf{Dynamic Orchestration.} We design a centralized \textit{load-balancing orchestrator} that dynamically re-partitions data-parallel groups across versions in proportion to their pending workloads and migrates requests to rebalance the cluster, eliminating resource fragmentation across coexisting policy versions.

\noindent$\bullet$ \textbf{Nearly Zero Re-prefill KV-Cache Reuse.} We design \textit{zero re-prefill migration} that transfers KV-Cache directly across same-version instances with negligible communication costs, eliminating prefill recomputation during request migration---especially beneficial for long-context reasoning and agentic scenarios.

\noindent$\bullet$ \textbf{Extensive Evaluation and Real-World Deployment.} On open-source benchmarks, DORA achieves up to $2.12\times$ end-to-end and $8.2\times$ rollout-stage speedup over synchronous training while preserving model convergence. \textit{Production deployment} with thousands of accelerators further yields up to $6.2\times$ rollout speedup, training a ${\sim}$500B-parameter MoE model competitive with 
state-of-the-art open-source LLMs.

\section{Preliminaries}\label{sec:preliminary}
\subsection{Asynchronous RL Training}
\label{sec:async_rl}

RL post-training for LLMs proceeds in iterative steps, each consisting of three stages: \emph{rollout} (sampling responses from the current policy), \emph{experience preparation} (computing rewards and references), and \emph{model training}. In synchronous training, a step cannot begin until all trajectories of the previous step are complete, enforcing a strict batch barrier between rollout and training. To overlap rollout and training, asynchronous methods relax this barrier and allow the training samples to mix trajectories generated under different behavior policy versions.

We take GRPO~\citep{shao2024deepseekmath}, a variant of PPO, as a representative algorithm. Given a prompt $x$, $G$ trajectories $\{y_i\}_{i=1}^G$ are sampled per prompt and used to update the policy $\pi_\theta$ via a clipped importance-weighted objective with group-relative advantages $\hat{A}_i$ shared across all tokens of trajectory $y_i$~\citep{guo2025deepseek,yu2025dapo}. Let $v(\cdot)$ denote the version index of a policy and $K$ a configurable upper bound on staleness. For trajectory $y_i$ under behavior policy $\pi_{w_i}$ updating training policy $\pi_\theta$, asynchronous training requires:
\begin{equation}
\label{eq:staleness}
v(\theta) - v(w_i) \;\leq\; K,
\end{equation}
which is the standard condition for convergence guarantees in asynchronous optimization~\citep{lian2015asynchronous,zheng2017asynchronous}. The asynchronous GRPO objective replaces the single behavior policy with the per-trajectory $\pi_{w_i}$ in the importance ratio:
\begin{equation}
\label{eq:grpo_async}
\mathcal{J}_\text{async}(\theta) =
\mathbb{E}\Bigg[
\frac{1}{G}\sum_{i=1}^G \frac{1}{L_i}\sum_{t=1}^{L_i}
\min\!\Big(
r_{i,t}(\theta)\hat{A}_{i,t},\;
\operatorname{clip}_\varepsilon\!\big(r_{i,t}(\theta)\big)\hat{A}_{i,t}
\Big)
\Bigg],
\quad
r_{i,t}(\theta)=\frac{\pi_\theta(y_{i,t}\mid\cdot)}{\pi_{w_i}(y_{i,t}\mid\cdot)},
\end{equation}
subject to Equation~\ref{eq:staleness}. Throughout the paper, we denote the rollout batch size (number of prompts dispatched per step) as $RBS$ and the training batch size (number of trajectories consumed by the training stage) as $TBS$.

\subsection{Long-Tail Dilemma}
\label{sec:longtail}
The core objective of asynchronous RL is to resolve the massive hardware idleness caused by long-tailed generation, which manifests at the system level through two compounding factors:

\begin{figure*}[t]
\begin{minipage}[t]{0.32\linewidth}
\includegraphics[width=1\textwidth]{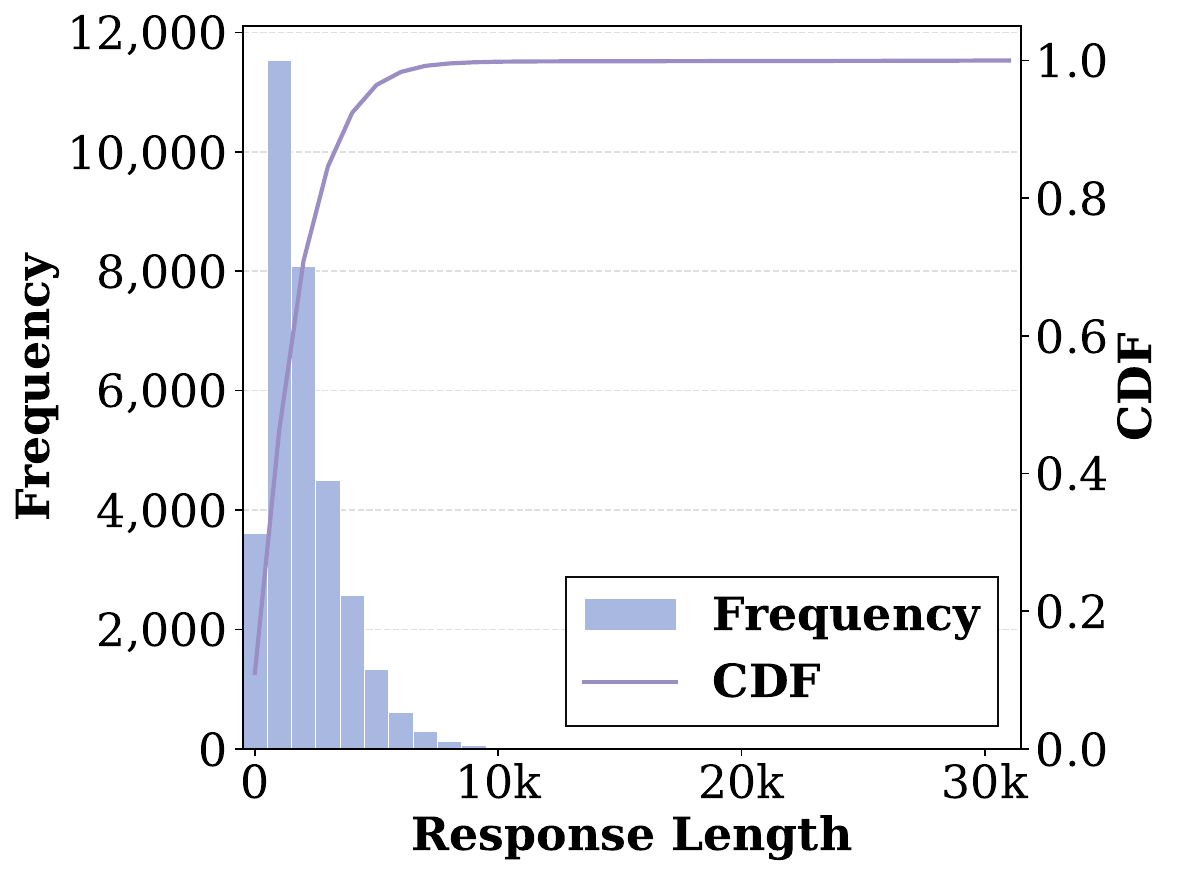}
\caption{Response length distribution on the DAPO-Math-17K dataset.}
\label{fig:resp_in_house}
\end{minipage}\hfill
\begin{minipage}[t]{0.32\linewidth}
\includegraphics[width=1\textwidth]{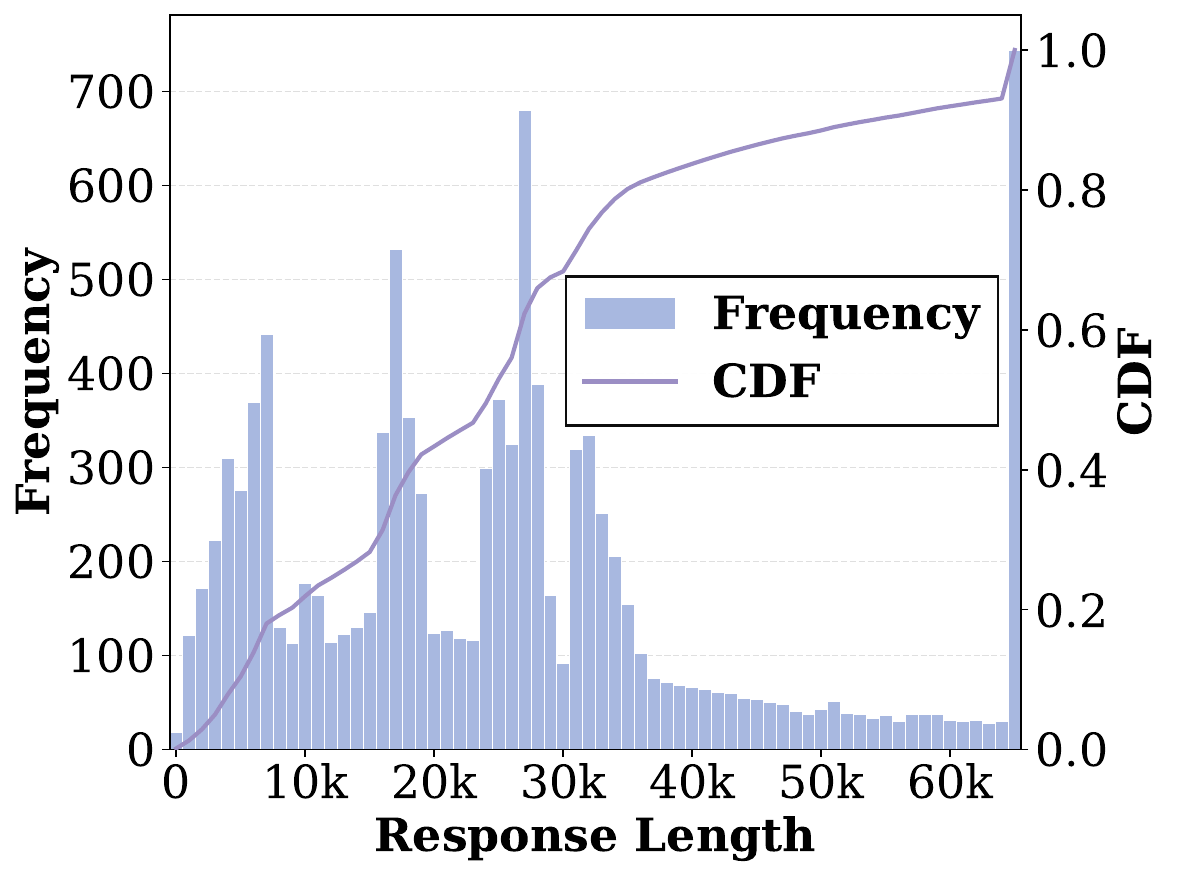}
\caption{Response length distribution in production. Stacked bar in the end of x-axis reflects overlong truncation.}
\label{fig:resp_in_house_2}
\end{minipage}\hfill
\begin{minipage}[t]{0.32\linewidth}
\includegraphics[width=1\textwidth]{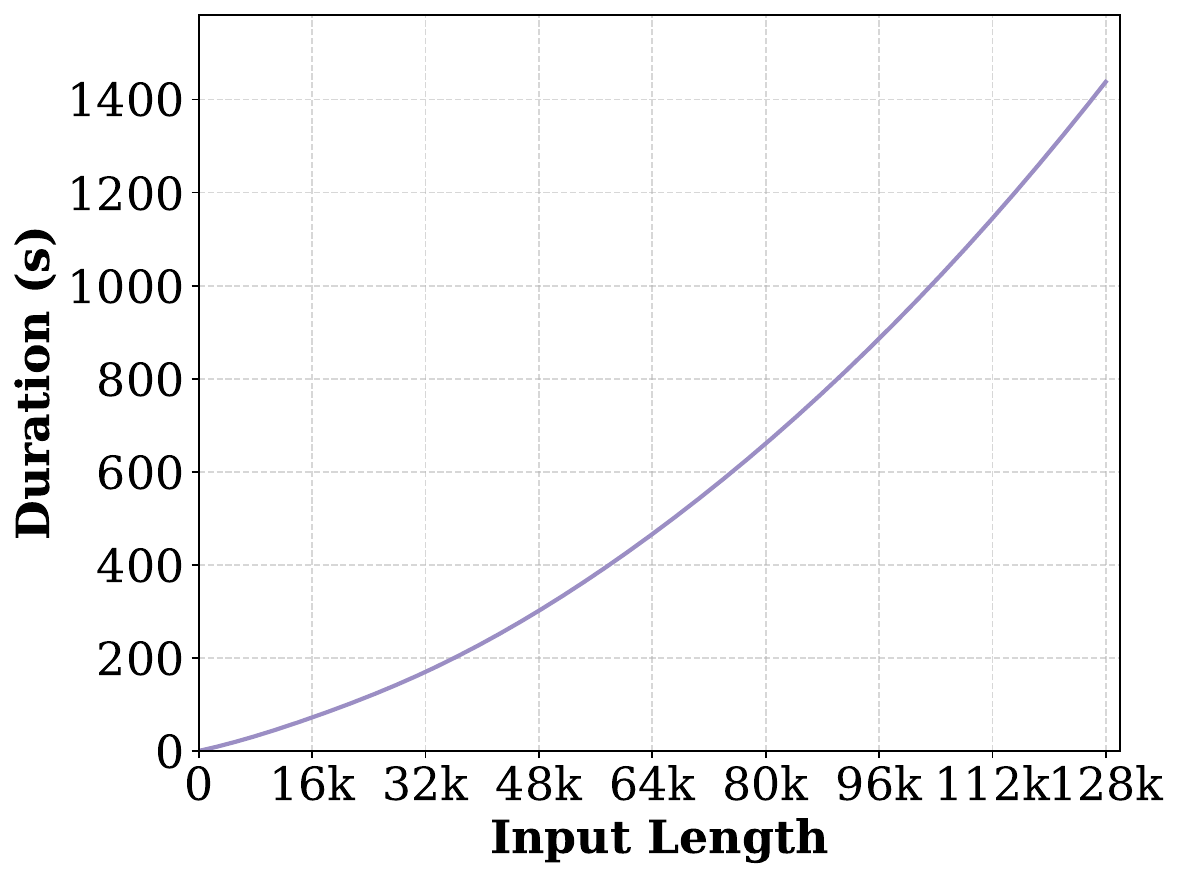}
\caption{Prefill duration of a ${\sim}$500B MoE model with expert parallelism size 128 on non-CUDA accelerators.}
\label{fig:prompt_in_house}
\end{minipage}
\end{figure*}

\textbf{Workload skewness.} In long-context reasoning workloads, response lengths $L_i$ follow a heavily long-tailed distribution. As shown in Figures~\ref{fig:resp_in_house} and~\ref{fig:resp_in_house_2}, the $99$th-percentile output exceeds the median by over an order of magnitude on both an open-source benchmark and a production workload. Because decode is memory-bound, this tail cannot be flattened by adding compute. Prefill cost compounds this issue: as shown in Figure~\ref{fig:prompt_in_house}, prefill duration grows dramatically with input length, making any technique that re-triggers prefill mid-step (e.g., re-prefill after weight updates) increasingly expensive in long-context regimes.

\begin{figure*}[t]
\begin{minipage}[t]{0.48\linewidth}
\includegraphics[width=1\textwidth]{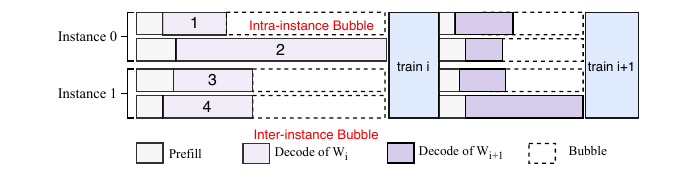}
\caption{Skewed bubbles in synchronous training: intra-node bubble (idle slots within a device) and inter-node bubble (faster instances waiting for the slowest).}
\label{fig:bubble_arch}
\end{minipage}\hfill
\begin{minipage}[t]{0.48\linewidth}
\includegraphics[width=1\textwidth]{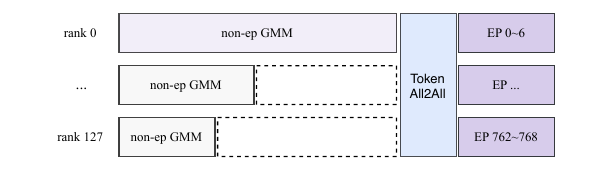}
\caption{Non-EP GMM workloads become unbalanced under long-tailed inputs (expert parallelism size $128$).}
\label{fig:moe_imbalance}
\end{minipage}
\end{figure*}

\textbf{Hardware-level idleness.} The rollout phase processes $RBS$ prompts per step, where each request incurs a compute-bound prefill and a memory-bound decode, with concurrency capped by accelerator memory occupied by model weights and KV-Cache. The objective is to minimize total step duration:

\begin{equation} \label{eq1}
    \min \; T_{\text{train}} + T_{\text{Prefill}} + T_{\text{Decode}}
    \;\Rightarrow\;
    \min \; T_{\text{train}} + T_{\text{Prefill}} +
    \overbrace{
        \tau \max_{j}
        \underbrace{
            \max_{i\in \text{Device}_{j}}\{L_{i}\}
        }_{\text{intra-node bubble}}
    }^{\text{inter-node bubble}}
\end{equation}
where $\tau$ is the time per output token (TPOT, approximately constant under fixed batch size for rollout engine). Combined with the long-tailed distribution above, this $\max\max$ structure in the rollout creates two forms of device idleness illustrated in Figure~\ref{fig:bubble_arch}: an \emph{intra-node bubble}, where completed slots on a device sit idle while a long-tailed request continues without device saturation, and an \emph{inter-node bubble}, where faster instances wait for the slowest one. In MoE architectures, this skew further propagates into non-EP layers, where the slowest rank stalls the entire EP group (Figure~\ref{fig:moe_imbalance}).

While long-tail trajectories bottleneck rollout efficiency, their rich learning signals make them indispensable for RL. This \emph{long-tail dilemma} exposes a fundamental conflict between hardware utilization and algorithmic integrity. Consequently, existing asynchronous methods cannot mitigate rollout bubbles without incurring \textcolor{black}{significant} re-prefill overhead or algorithmic sampling bias (Section~\ref{sec:intro}, Appendix~\ref{app:related_work}). The remainder of this paper introduces a system architecture that natively resolves this tension, \textcolor{black}{eliminating bubbles without incurring significant system overhead or sacrificing algorithmic fidelity.}




\section{DORA Design}\label{sec:methods}
\subsection{System Overview}
Without loss of generality, we present DORA using a disaggregated architecture, although it readily extends to the colocated architecture. As motivated in Section~\ref{sec:longtail}, our goal is to resolve the long-tail dilemma without paying either significant system overhead or additional algorithmic cost. DORA achieves this through three interlocking mechanisms:

\noindent\emph{(i) Multi-version streaming training} (Section~\ref{sec:Multi-version Asynchronous Training}) maintains multiple policy versions concurrently on rollout instances, enabling trajectory-level streaming that eliminates both intra-node and inter-node bubbles. Each trajectory is generated end-to-end under a single policy version, and the staleness across versions is bounded by a configurable window.

\noindent\emph{(ii) Dynamic resource orchestration} (Section~\ref{sec:load_balancing}) resolves the resource fragmentation inherent in multi-version rollout by dynamically re-partitioning DP groups and migrating requests, while preserving every sampled trajectory and respecting the staleness bound throughout.

\noindent\emph{(iii) KV-Cache reuse} (Section~\ref{sec:KV Cache Reuse}) turns the single-policy-per-trajectory design into a system-level advantage: because every token of a trajectory is generated under the same policy version, its KV-Cache states are mathematically equivalent across any instance hosting that version, enabling nearly zero-re-prefill migration during request relocation.

As shown in Figure~\ref{fig:dora_workflow}, these mechanisms cooperate at runtime through four cooperating components. A \texttt{RolloutManager} dispatches prompts to rollout instances, tagging each prompt with a policy version so that the entire trajectory is generated under that version. Completed trajectories stream into an asynchronous \texttt{TransferQueue} equipped with staleness monitoring. The \texttt{Trainer} consumes $TBS$ samples for experience preparation and model training, then synchronizes the latest weights with rollout instances. A \texttt{Load-balancing orchestrator} monitors per-version workloads and triggers resource re-partitioning and request migration as needed, preserving the intermediate execution state. These components run on different nodes and coordinate via Remote Procedure Call (RPC), while workers running on accelerators execute the actual tasks.

Together, these mechanisms eliminate the long-tail bubble in rollout while preserving every sampled trajectory and bounding policy staleness—\textcolor{black}{resolving the long-tail dilemma without significant system overhead or algorithmic compromise.}

\begin{figure}[t]
  \centering
  \includegraphics[width=0.70\textwidth]{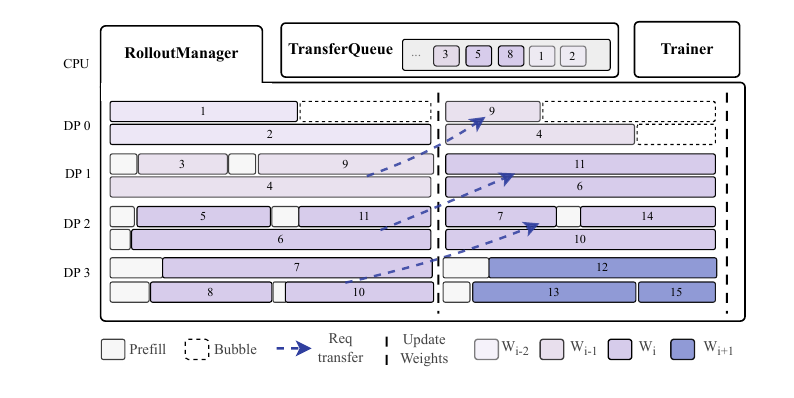}
  \vspace{-0.2cm}
  \caption{The execution timeline of DORA's multi-version streaming training system.}
    \label{fig:dora_workflow}
\end{figure}

\subsection{Multi-version Streaming Training}
\label{sec:Multi-version Asynchronous Training}

\textbf{Trajectory-level streaming.}
DORA eliminates the synchronous barrier by streaming completed trajectories directly to training without waiting for straggling trajectories. During the rollout phase, we maintain multiple versions of policy weights in the rollout instances, where each Data Parallel (DP) group hosts a single version of the policy weights. At the onset of each step, we overprovide the rollout prompts, where $RBS > TBS$ generation requests are dispatched to the rollout instances. Training begins as soon as $TBS$ samples are collected; unfinished long-tailed trajectories continue under their original policy version and flow into subsequent steps, ensuring the long trajectories are not abandoned or block the training. As illustrated in Figure~\ref{fig:dora_workflow}, the Rollout and training processes execute non-blockingly; only after a training iteration concludes does the \texttt{Trainer} notify the \texttt{RolloutManager} to synchronize the latest weights.

\textbf{Multi-version policy management.}
The key insight is that maintaining multiple policy versions concurrently allows long-tailed trajectories to continue under their original version while the training proceeds with completed trajectories. As illustrated in Figure~\ref{fig:dora_workflow}, each prompt is tagged with a version $w_j$ upon dispatch, ensuring $a_t \sim \pi_{w_j}(\cdot \mid s_t)$ for every token---so each trajectory is generated end-to-end under a single policy version without algorithmic modifications. For example, in Figure~\ref{fig:dora_workflow}, Trajectory~4 (a long-tail request under version $w_1$) spans two training steps while Trajectories~1--3 complete and stream into training during Step~1. The system proceeds to Step~2 with updated weights $w_2$ without waiting for Trajectory~4 to complete, which continues under its original version $w_1$ in a dedicated DP group. This allows the legacy-version requests to execute in parallel, thereby eliminating the inter-node bubble and fully utilizing all rollout instances.

\textbf{Sliding-window staleness control.}
Active versions are managed through a sliding window $W = \{w_{j}, \ldots, w_{j-K+1}\}$ of size $|W| \leq K$. The window advancement follows a strict protocol to control the staleness. The window slides forward only when all trajectories from the oldest version $w_{j-K+1}$ have been collected and forwarded to training. This provides a deterministic upper bound on policy staleness. The staleness bound $K$ serves as an explicit control knob for the convergence--throughput tradeoff: a smaller $K$ yields more on-policy data at the cost of rollout efficiency; a larger $K$ increases throughput with controlled convergence impact.

\textbf{Remaining challenges.}
While multi-version streaming preserves every sampled trajectory, bounds policy staleness, and partially alleviates both bubble types, it introduces two second-order efficiency challenges that motivate the subsequent mechanisms:

\emph{(i) Resource fragmentation.} \textcolor{black}{We observe that the legacy version's pending requests decrease monotonically as trajectories complete, yet its allocated resources remain fixed, leading to notable device underutilization.} This motivates the dynamic orchestration described in Section~\ref{sec:load_balancing}.

\emph{(ii) Re-prefill overhead.} Migrating requests across DP groups naively re-triggers the full prefill phase---prohibitive for long-context scenarios (64k--128k tokens). This motivates the KV-Cache reuse mechanism in Section~\ref{sec:KV Cache Reuse}, which exploits the mathematical equivalence guaranteed by single-version trajectory generation.

\begin{figure}[t]
  \centering
  \includegraphics[width=0.8\textwidth]{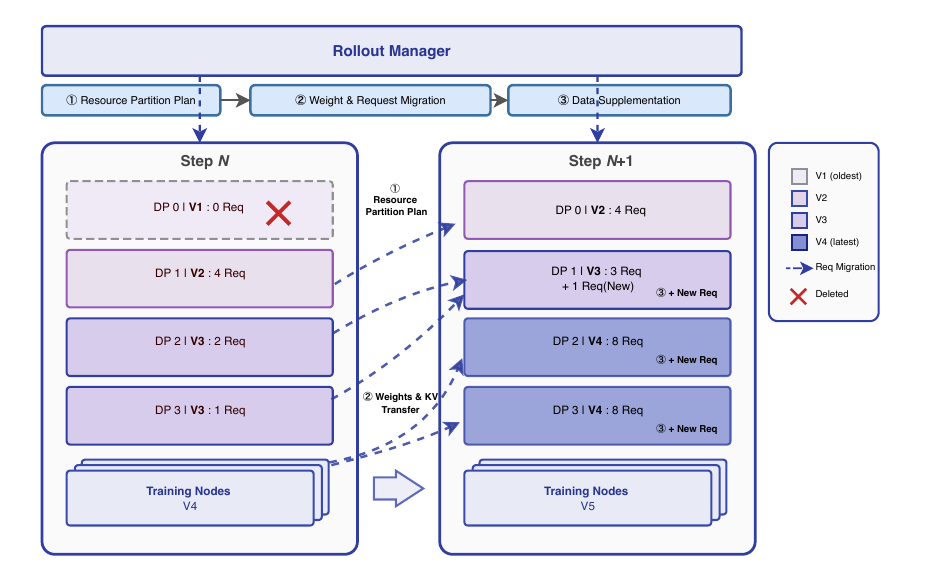}
  \vspace{-0.2cm}
  \caption{The workflow of the Dynamic Resource Orchestration.}
    \label{fig:load_balance}
\end{figure}

\subsection{Dynamic Resource Orchestration}
\label{sec:load_balancing}
Under multi-version streaming training, the pending request count of each legacy version $w$ decreases monotonically as trajectories complete, yet the physical resources (DP groups) assigned to $w$ remain fixed. This implies that static resource allocation leads to progressive underutilization---a legacy version's DP groups may each serve only one or two residual requests, while the latest version, which carries the majority of new prompts, is under-provisioned. Therefore, DORA requires proactively rebalancing workloads while simultaneously controlling staleness.

To resolve this resource fragmentation, DORA employs a centralized orchestrator that dynamically re-partitions resources across model versions. The orchestrator maintains real-time metrics---active request counts per version, KV-Cache utilization, and generation progress---and supports three re-balancing triggers: (1)~\textit{update-driven}, mandatory upon the completion of each training step to promote the new policy version; (2)~\textit{utilization-based}, activated when KV-Cache pressure exceeds a threshold, avoiding costly eviction-induced recomputation; and (3)~\textit{temporal-based}, periodic execution to prevent orphan requests from lingering in legacy versions.

As illustrated in Figure~\ref{fig:load_balance}, each re-balancing cycle executes three coordinated operations:
\begin{itemize}[leftmargin=1.2em, itemsep=2pt]
    \item \textbf{Resource partitioning plan.} The orchestrator assesses the distribution of active and pending requests across all maintained versions $W$ to produce a migration plan. It computes the target DP group count for each version $w \in W$ proportional to its current workload, preventing resources from being stranded on legacy versions with dwindling tasks. This addresses the inter-instance data skewness identified in Section~\ref{sec:longtail}.
    \item \textbf{P2P weight update and request migration.} Once the migration plan is determined, the orchestrator generates a mapping from the current partition to the target partition. For each version whose allocation changes, it leverages P2P weight transfers to rescale the DP groups---decrease the legacy versions and increase the latest one. Active requests on re-assigned nodes are migrated to their new DP groups with execution states fully preserved via KV-Cache reuse (Section~\ref{sec:KV Cache Reuse}). Notably, no trajectory is abandoned during this process.
    \item \textbf{Staleness-aware data supplementation.} To control data staleness within the configured bound, the orchestrator prioritizes the latest policy version for proactive data injection, dispatching supplemental prompts until the $RBS$ is fully met. This strategy maximizes sample freshness by ensuring that the majority of new trajectories are generated using the most recent model weights. Subsequently, to maintain high-watermark utilization across the cluster, the orchestrator performs opportunistic data injection following the request migration phase. Legacy versions are only supplemented with sufficient prompts to fill their residual idle slots. This tiered injection approach effectively saturates all rollout instances while preventing the over-production of stale trajectories, striking an optimal balance between hardware occupancy and algorithmic freshness.
\end{itemize}
This optimization cycle exemplifies algorithm-system co-design: \textbf{Proportional Resource Partitioning} resolves the resource fragmentation inherent in multi-version rollout; \textbf{Request Migration} preserves trajectory execution state without abandoning any sampled trajectory; and \textbf{Staleness-Aware Data Supplementation} ensures high hardware occupancy within the staleness bound.


\subsection{KV-Cache Reuse}
\label{sec:KV Cache Reuse}
\textcolor{black}{Note that request migration} across DP groups naively re-triggers the prefill phase. The re-prefill cost scales with the \textcolor{black}{growing} context length and is further amplified in MoE architectures~\citep{jiang2024mixtral,guo2025deepseek}, where it causes workload imbalance across non-MoE layers as shown in Section~\ref{sec:longtail}.
However, DORA's single-policy-per-trajectory design enables a powerful system-level optimization: since all tokens in a trajectory are generated by the same policy version $\pi_{w}$, the KV-Cache states are \emph{mathematically equivalent} across any physical instance hosting version $w$. This equivalence enables cross-instance KV-Cache transfer that completely avoids re-prefill. Methods that mix multiple policy versions within a single trajectory forfeit this optimization: each weight update forces a full re-prefill of all ongoing trajectories, up to the output length.
When the \texttt{Load-Balancing Orchestrator} triggers a resource re-allocation, DORA executes a coordinated two-phase state transfer:
\begin{itemize}[leftmargin=1.2em, itemsep=2pt]
    \item \textbf{Metadata forwarding.} Request metadata (request ID, generation state, decoded token count, and version tag) is transmitted via lightweight RPC. This control-plane transfer is negligible in both latency and bandwidth.
    \item \textbf{KV-Cache data transfer.} The voluminous KV Cache data—often comprising tens of gigabytes for long-context and MoE settings—is transferred using high-performance collective communication primitives, fully exploiting the available interconnect bandwidth.
\end{itemize}
\textbf{Locality-aware scheduling.}
To minimize transfer volume, the orchestrator prioritizes re-assigning requests back to their original ranks when possible---preserving data locality and avoiding physical migration entirely. Only requests that must relocate due to version transitions incur transfer costs.

\textbf{Hierarchical memory management.}
To alleviate VRAM pressure from aggregated requests during long-context training, DORA temporarily offloads KV-Caches to host memory~\citep{qin2024mooncake}, freeing device memory for active computations while preserving state for deferred generation. This hierarchical management safeguards system efficiency even under extreme long-tailed workloads.

By eliminating re-prefill from request migration, KV-Cache reuse closes the last source of system overhead introduced by multi-version coexistence: long-tailed trajectories now traverse legacy DP groups, get migrated under load-balancing, and continue generation---all without the prefill recomputation that would otherwise \textcolor{black}{grow} dramatically with context length. 

\section{Experiments}\label{sec:exp}

\begin{figure*}[t]
    \begin{minipage}[t]{0.33\linewidth}
        \includegraphics[width=1\textwidth]{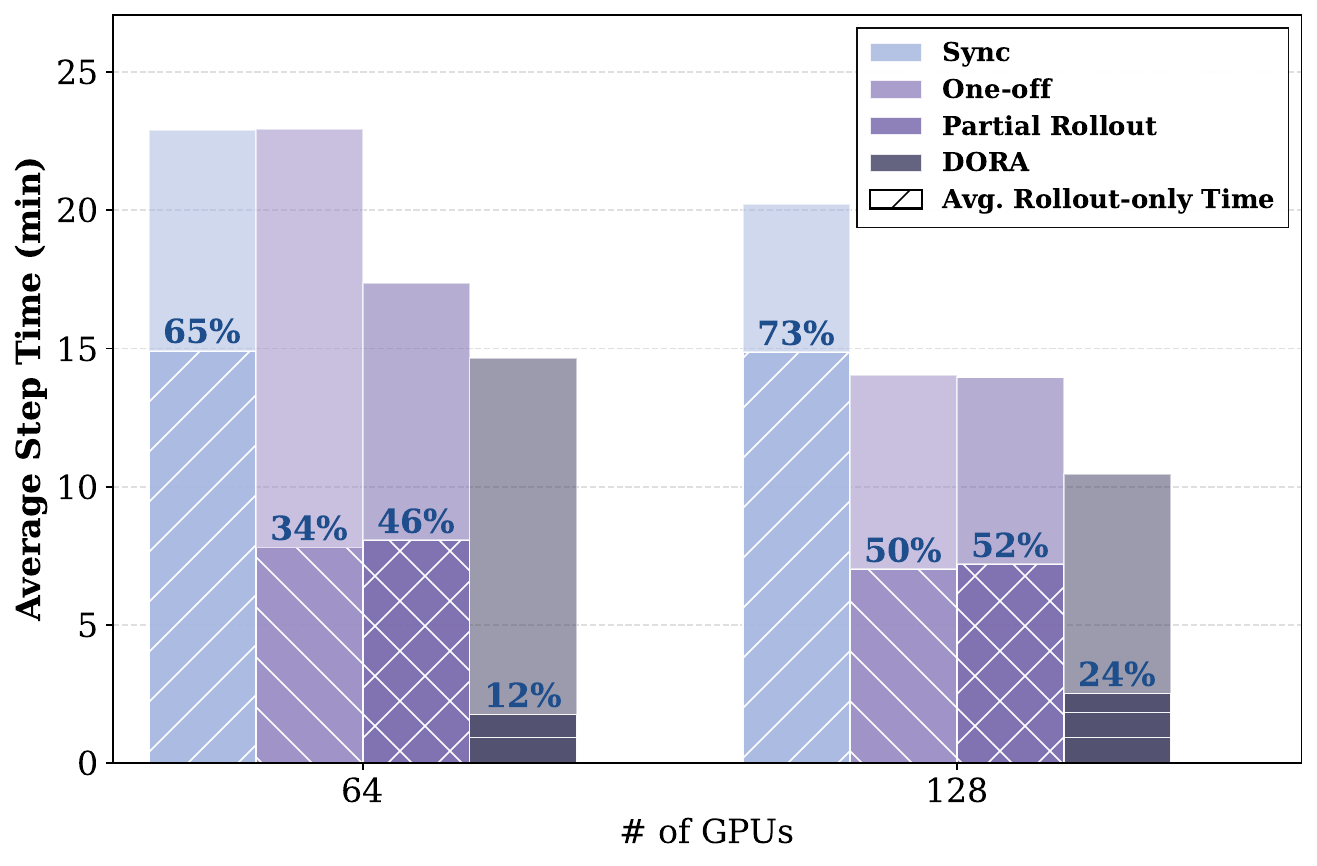}
        \captionsetup{width=0.95\textwidth}
        \caption{Average RL step time across different training paradigms on Dense-32B. }
        \label{fig:step_time}
    \end{minipage}
    \begin{minipage}[t]{0.33\linewidth}
        \includegraphics[width=1\textwidth]{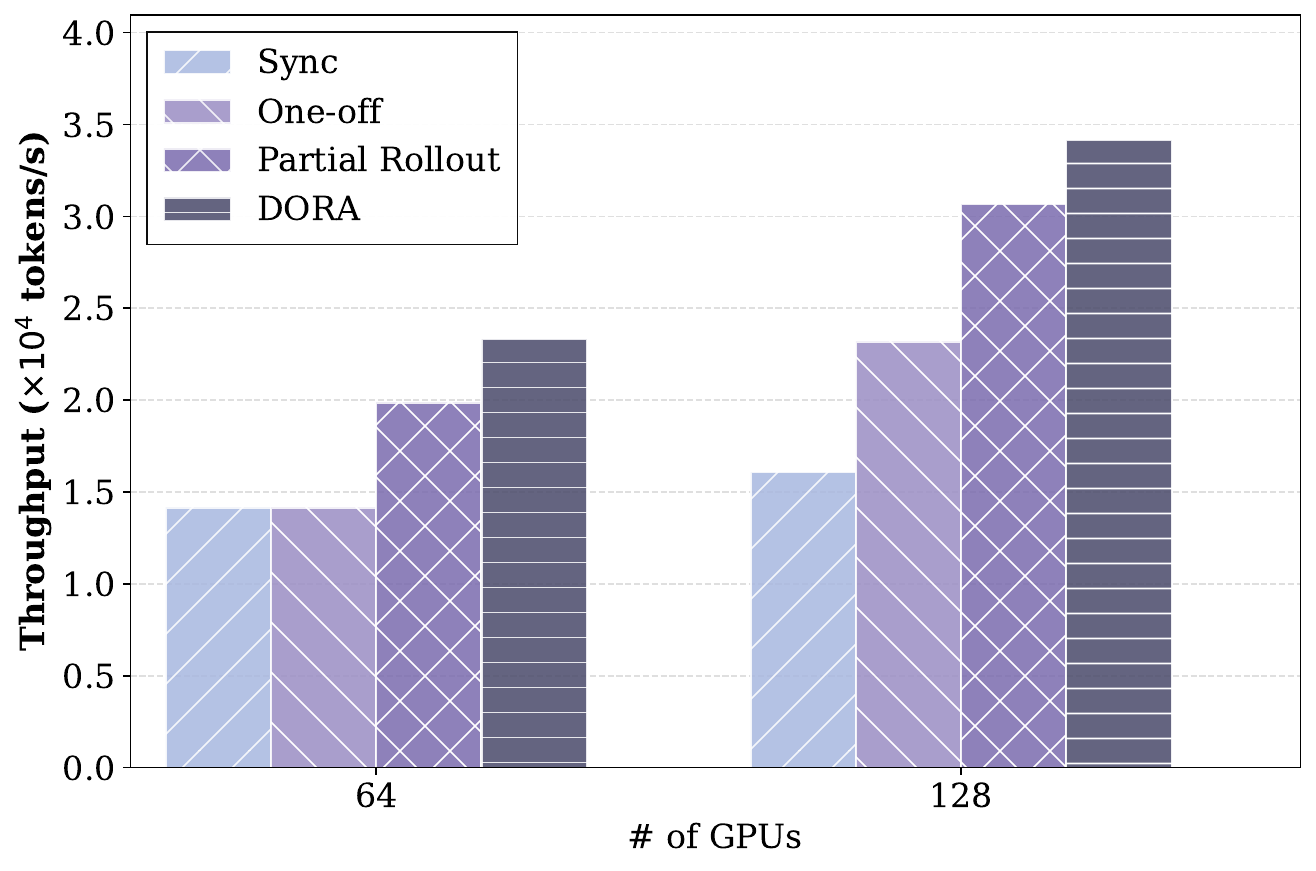}
        \captionsetup{width=0.95\textwidth}
        \caption{End-to-End Throughput across different training paradigms on Dense-32B.}
        \label{fig:throughput}
    \end{minipage}
        \begin{minipage}[t]{0.33\linewidth}
        \includegraphics[width=1\textwidth]{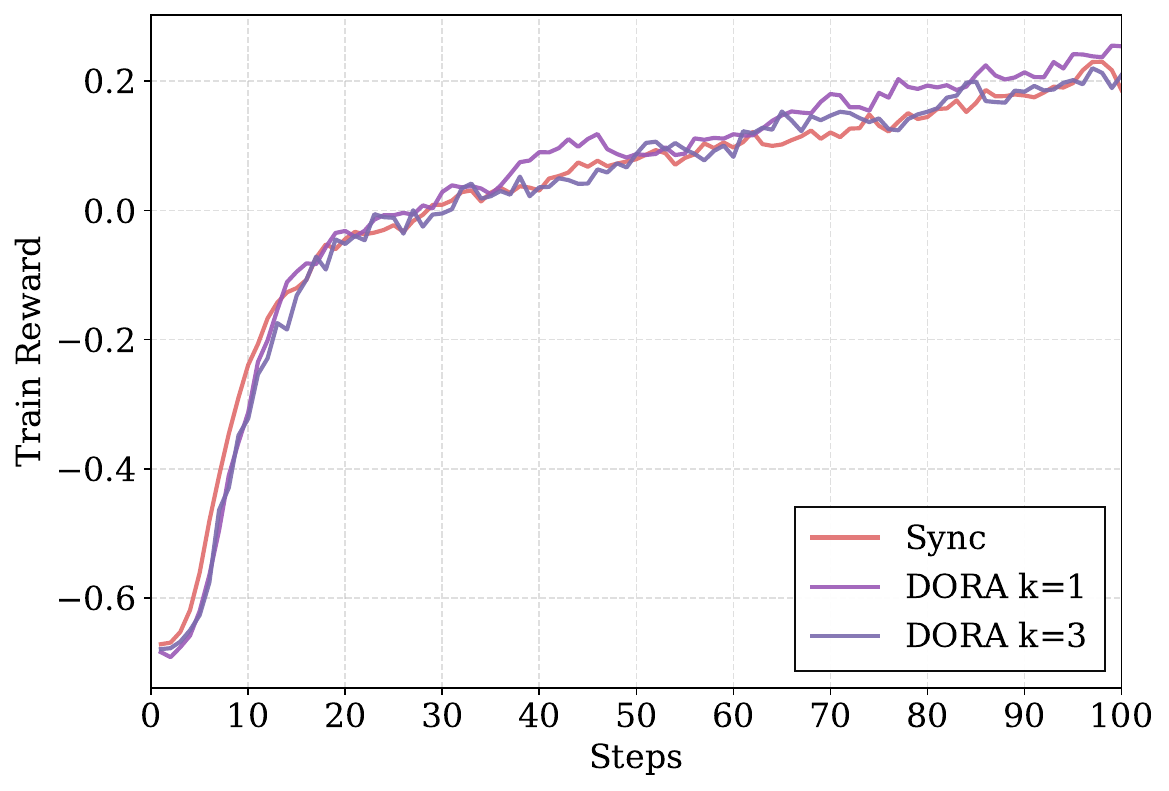}
        \caption{Training reward scores comparison for various training paradigms.}
        \label{fig:convergence}
    \end{minipage}
\end{figure*}

\subsection{Experimental Setup}
We evaluate DORA on a 16-node H800 cluster (128 GPUs) for open-source experiments and on a production cluster of non-CUDA accelerators for large-scale evaluation. Two model scales are used: Qwen2.5-32B~\citep{qwen2025qwen25technicalreport} for dense architectures and a ${\sim}$500B-parameter MoE model for production scale. We compare DORA against three representative paradigms spanning the constraint--efficiency landscape: (1)~\textit{Synchronous} (all constraints satisfied, batch-barrier limited); (2)~\textit{One-step off-policy} (overlaps stages but does not eliminate rollout bubbles); and (3)~\textit{Partial rollout in the colocated model placement} (eliminates bubbles by relaxing single-version generation, requiring algorithmic corrections). All baselines are implemented in the same in-house framework on identical hardware. Full hardware specs, software stack, dataset, and training hyperparameters are in Appendix~\ref{app:setup}.

\subsection{Training Performance}
\label{sec:training_efficiency}
\label{sec:Training Time}
\label{sec:e2e Throughput}

 
        
\textbf{Rollout Acceleration.}
DORA achieves consistent gains across both 64 and 128 GPUs. As shown in Figure~\ref{fig:step_time}, on 64 GPUs, DORA shrinks the rollout-only phase—\textcolor{black}{the portion that cannot be overlapped with training}—from 65\% of the step time under synchronous training to merely 12\%, an \textbf{8.2$\times$} reduction in absolute duration (14.9\,min $\to$ 1.8\,min). The end-to-end step time correspondingly drops by \textbf{1.56$\times$}. This compression directly resolves the long-tail dilemma identified in Section~\ref{sec:intro}: long-tailed trajectories continue under their legacy policy versions in dedicated DP groups while new requests saturate the released resources, eliminating both the intra- and inter-node bubbles. Compared with partial rollout, the strongest long-tail-mitigating baseline, DORA still achieves \textbf{1.18$\times$} end-to-end speedup, owing to its zero-re-prefill migration (Section~\ref{sec:KV Cache Reuse}) that avoids the re-prefill cost partial rollout incurs at each weight update. The same pattern holds on 128 GPUs—Sync's rollout-only fraction grows to 73\% while DORA's remains at only 24\%—yielding \textbf{5.9$\times$} rollout and \textbf{1.93$\times$} end-to-end speedup\footnote{Theoretically, the rollout-only phase can be fully eliminated via aggressive staleness strategies \textcolor{black}{as detailed in Appendix~\ref{app:bubble}}.}.

\textbf{End-to-End Throughput.}
The step-time gains translate into proportional throughput improvements across both cluster scales. As shown in Figure~\ref{fig:throughput}, on 64 GPUs, DORA reaches \textbf{23{,}327\,tokens/s}, a \textbf{1.65$\times$} improvement over synchronous training and \textbf{1.17$\times$} over partial rollout. The same pattern holds on 128 GPUs, where DORA achieves \textbf{34{,}135\,tokens/s}, yielding \textbf{2.12$\times$} and \textbf{1.11$\times$} speedups over the two baselines, respectively. Notably, this throughput reflects the trajectories actually consumed by training. Both baselines waste accelerator time—\textcolor{black}{one-step off-policy} through long-tail idleness, partial rollout through re-prefill at each weight update—whereas DORA's multi-version streaming and zero-re-prefill migration convert all rollout duration into effective throughput.

\textbf{Model Convergence.}
To verify that DORA's efficiency gains do not compromise algorithmic fidelity, we monitor the mean training reward over 100 steps on 72 GPUs (Figure~\ref{fig:convergence}). Both DORA variants ($k{=}1$ and $k{=}3$) closely track the synchronous baseline, confirming that multi-version streaming training preserves convergence behavior in the bounded staleness settings. Nevertheless, we observe that the $k{=}3$ \textcolor{black}{variant} exhibits a moderately slower convergence rate compared to $k{=}1$, which further underscores the necessity of staleness control.

\subsection{Ablation Study}
We conduct an ablation study to further investigate the contribution of each component to DORA's overall acceleration. As shown in Figure~\ref{fig:ablation}, the baseline without KV cache reuse requires an average of 183 seconds per rollout, whereas the variant with KV cache reuse completes the same rollout in approximately 166 seconds under identical dense 32B model settings. This KV cache optimization yields a 9\% speedup over the non-cached baseline. We do not ablate the dynamic resource orchestration module, \textcolor{black}{as it is indispensable for the system to function.}

\subsection{Overhead Analysis}
We quantify the three primary overheads introduced by DORA's dynamic orchestration and KV-Cache reuse: P2P-based load balancing, request transfer, and free-cache operations. As shown in Figure~\ref{fig:overhead}, all three remain well below the throughput gains they enable. Load balancing—covering request monitoring, resource re-partitioning, and P2P weight synchronization—accounts for 0.4\% / 1.5\% of total execution time at 64 / 128 GPUs. Request transfer, which carries metadata and physical KV-Cache states, is bounded under 4\% and \emph{decreases} at scale (3.6\% $\to$ 2.1\%) as larger throughput amortizes the migration cost. Free-cache operations are negligible (under 0.03\% in both settings), confirming that our hierarchical memory management runs entirely off the critical path. Overall, the aggregate system overhead does not grow with cluster size.

\subsection{Production Deployment}
We further deploy DORA on a ${\sim}$500B-parameter MoE model with up to 64K-token responses, comparing against the well-tuned synchronous baseline used in production (\textcolor{black}{since} running all baselines at this scale is prohibitive) with \textcolor{black}{4,096} accelerators. As shown in Figure~\ref{fig:moe3b}, DORA achieves \textbf{3.6$\times$} rollout speedup on mathematical and tool-integrated reasoning, and up to \textbf{6.2$\times$} on agentic training over Tau2-bench~\citep{tau2-bench} and Vita~\citep{vita-bench}. The widening gap on agentic workloads—where responses are longest and most skewed—aligns with DORA's design hypothesis: the more pronounced the long tail, the larger the bubble that multi-version streaming eliminates. 
DORA has served as the default asynchronous paradigm in our in-house RL framework since 2025, delivering \textbf{2~--~4$\times$} end-to-end speedup with no quality degradation at scale, powering our competitive open-source LLMs.

\begin{figure*}[t]
    \begin{minipage}[t]{0.32\linewidth}
        \centering 
        \includegraphics[width=1\textwidth]{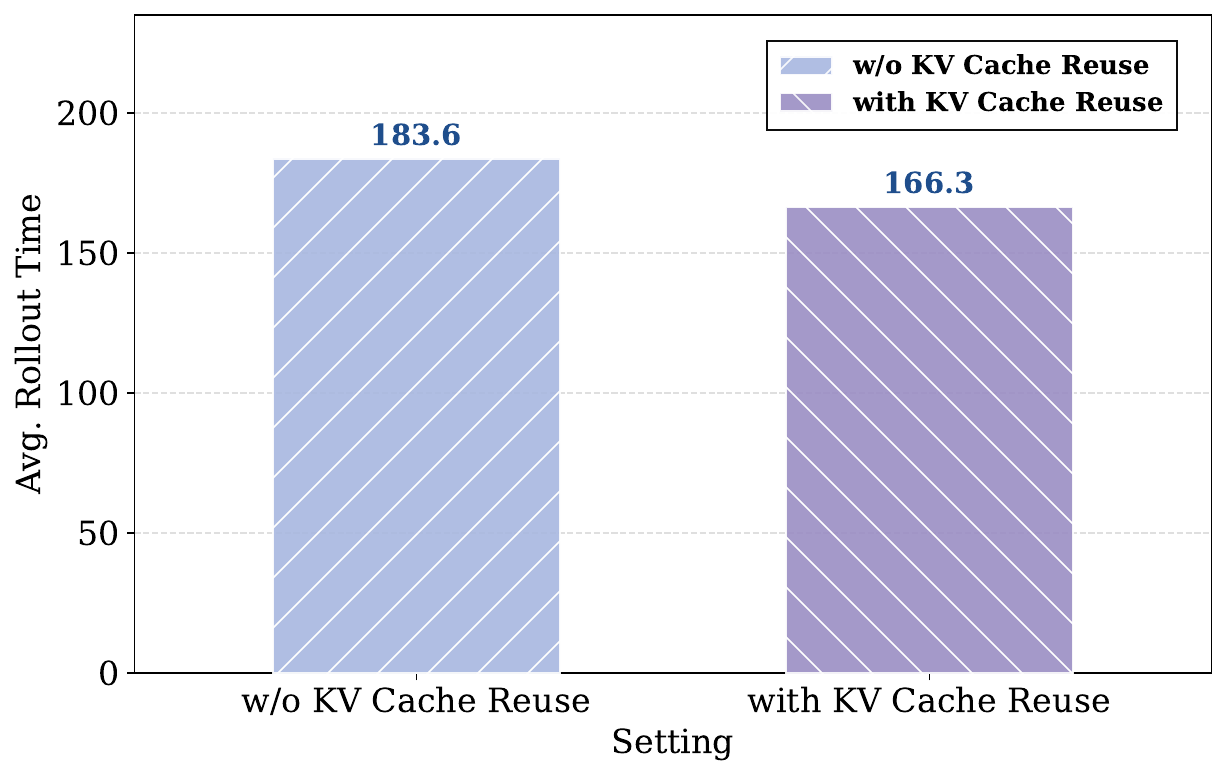}
        \caption{Ablation study on the effect of KV Cache reuse on average rollout time (in seconds).}
        \label{fig:ablation}
    \end{minipage}
    \hfill 
    \begin{minipage}[t]{0.32\linewidth}
        \centering
        \includegraphics[width=1\textwidth]{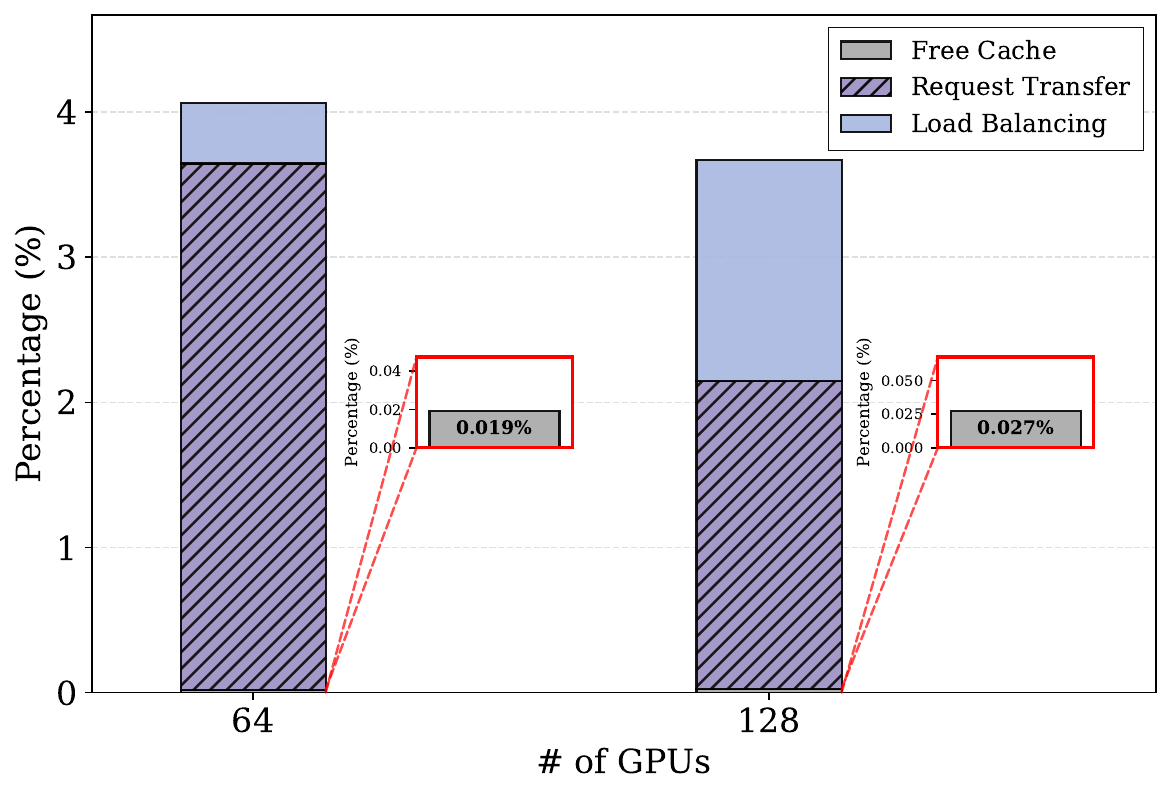} 
        \caption{System overhead breakdown for DORA on 64 and 128 GPUs.}
        \label{fig:overhead}
    \end{minipage}
    \hfill 
    \begin{minipage}[t]{0.32\linewidth}
        \centering
        \includegraphics[width=1\textwidth]{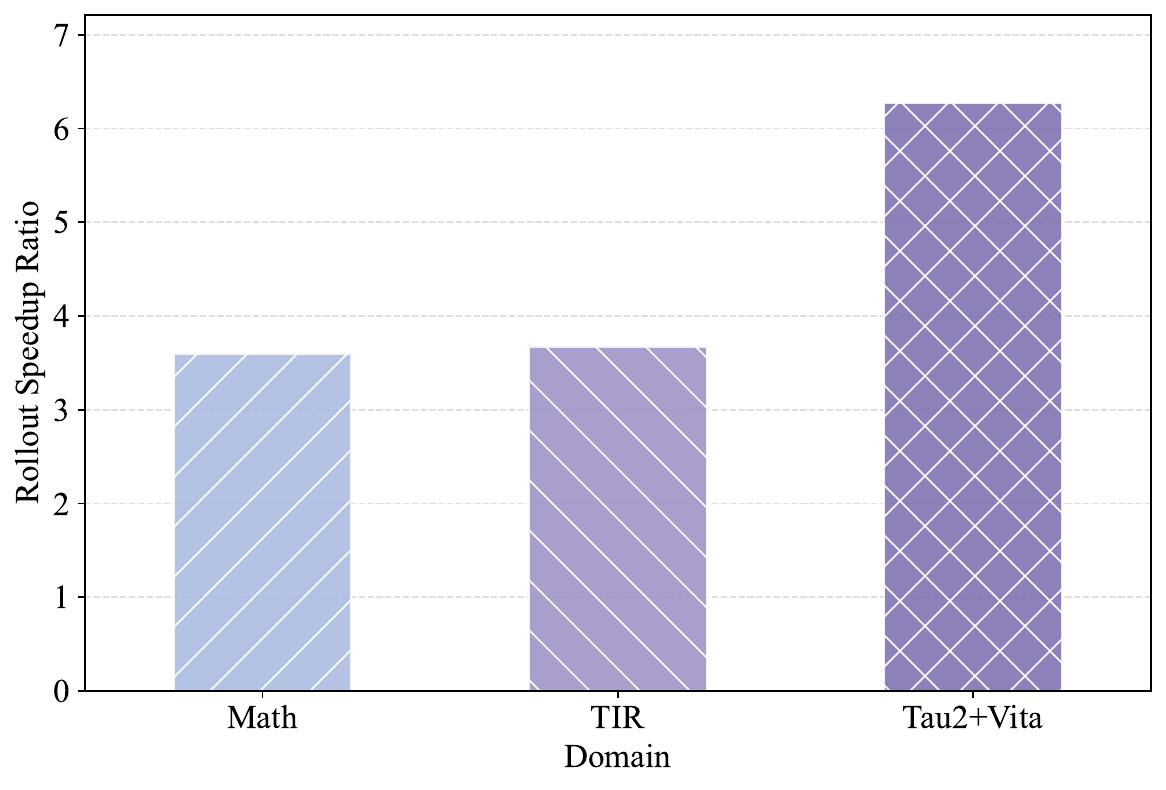}
        \caption{DORA vs. sync in production on a ${\sim}$500B MoE model with 64K max resp. len.}
        \label{fig:moe3b}
    \end{minipage}
\end{figure*}

\section{Conclusion and Limitations}
We present DORA, a scalable asynchronous RL system that resolves the long-tail dilemma in large-scale LLM post-training. By embodying a \textit{multi-version rollout} paradigm, where multiple policy versions coexist and each trajectory is generated end-to-end under a single version, DORA eliminates the long-tail bubble without introducing algorithmic corrections. Furthermore, this single-policy generation yields KV-Cache equivalence, enabling zero-re-prefill migration during request relocation. Experiments demonstrate up to $8.2\times$ rollout speedup and $2.12\times$ end-to-end acceleration over synchronous training, validated by large-scale industrial deployments. Despite these significant efficiency gains, DORA has limitations that merit further investigation. Algorithmically, the staleness bound $K$ requires manual configuration, relying on PPO's clipping to mitigate off-policy bias; incorporating adaptive staleness control or explicit delay compensation could further optimize the convergence--throughput tradeoff. Experimentally, our evaluations were conducted within a controlled in-house framework, and MoE scaling was validated primarily on production data and cluster. Future efforts will address this by benchmarking directly against public systems such as veRL~\citep{sheng2025hybridflow} and AReaL~\citep{fu2025areal}, alongside extensive evaluations on more open-source MoE models.




{
\small
\bibliographystyle{plainnat}
\bibliography{main}
}

\newpage
\appendix
\section{Related Work}\label{app:related_work}

To accelerate reinforcement learning (RL) post-training for large language models, various distributed training systems and asynchronous strategies have been proposed. Previous RL training systems~\citep{yao2023deepspeed,xiao2023adaptive,hu2024openrlhf,sheng2025hybridflow} primarily focus on model placement and scheduling within a synchronous training paradigm. Synchronous training cycles sequentially through rollout, experience preparation, and model training within each step. While this provides clean algorithmic semantics by enforcing a strict batch barrier between rollout and training, it suffers from severe batch-barrier idle time—especially in long-context and complex reasoning scenarios where the generation length is highly skewed. To alleviate this rollout bottleneck, recent efforts have shifted towards asynchronous and semi-asynchronous paradigms, which can be broadly categorized into three main directions. 

\paragraph{Replication-based (Oversampling) Methods.}
Replication-based or oversampling methods~\citep{gao2025rollpacker,zhang2025sortedrl} attempt to shorten the rollout duration by over-provisioning generation prompts. Specifically, they dispatch a larger number of rollout requests than the required training batch size and simply discard the in-flight long trajectories once enough completed responses are collected. While this approach effectively reduces the long-tail latency, it fundamentally compromises data integrity. The discarded trajectories often contain the lengthy chain-of-thought reasoning steps that are crucial for developing emergent agentic capabilities. Furthermore, the resulting length-biased distribution distorts advantage estimation, particularly in group-relative algorithms like GRPO~\citep{shao2024deepseekmath}, whereas DORA preserves every sampled trajectory without algorithmic distortion.

\paragraph{One-step Off-policy Methods.}
$K$-step off-policy methods~\citep{luo2025deepcoder,zhong2025streamrl,noukhovitch2024asynchronous,he2025history,han2025asyncflow}—most commonly one-step off-policy—overlap the rollout and training stages temporally. They achieve this by allowing the current step's rollout to use behavior policy weights from the previous iteration while the training stage updates the target policy. Although this design successfully hides the pipeline bubble between the rollout and training stages, it does not address the fundamental duration of the rollout phase itself. Both intra-node bubbles (idle slots within a device) and inter-node bubbles (faster instances waiting for the slowest) persist under highly skewed workload distributions. In contrast, DORA utilizes a multi-version streaming mechanism that completely breaks the synchronous batch barrier, directly resolving these intra- and inter-node inefficiencies.

\paragraph{Partial-rollout Methods.}
Partial-rollout methods~\citep{team2025kimi,wu2025llamarl,du2025ulorl,fu2025areal,slime_github} mitigate the long-tail issue by segmenting lengthy responses into smaller chunks at each weight update. When a new policy version becomes available, ongoing trajectories are resumed and continued under the latest weights. This approach effectively eliminates device idleness but introduces two significant challenges. First, from an algorithmic perspective, a single trajectory is now stitched together from multiple distinct policy versions. This departs from the standard RL formulation and necessitates complex algorithmic corrections, such as masking earlier segments during loss computation~\citep{team2025kimi} or applying decoupled PPO objectives~\citep{fu2025areal} to maintain convergence. Second, from a system perspective, every weight update invalidates the KV-Cache for the ongoing trajectories, forcing a full re-prefill of the context. This re-prefill overhead grows dramatically with context length and is particularly prohibitive in MoE architectures. DORA bypasses both challenges by ensuring each trajectory is generated end-to-end under a single policy version, preserving standard algorithmic semantics and enabling zero-re-prefill migration.

Concurrently with our work, several systems~\citep{seed2025seed1,sheng2025laminar} have explored multi-version streaming training concepts. However, these systems rely on two-tier CPU relay architectures for weight management and employ significantly different mechanisms for workload orchestration and staleness control compared to DORA's centralized load-balancing orchestrator and zero-re-prefill migration.

\section{Experimental Setup}\label{app:setup}
\textbf{Testbed.} Our experiments are conducted on a cluster consisting of 16 nodes, each equipped with 8 NVIDIA H800 GPUs. Intra-node communication is facilitated by NVLink with a bandwidth of 400 GB/s, while inter-node connectivity is provided by 8$\times$400 Gbps network interfaces. Additionally, our production cluster employs non-CUDA accelerators, each providing approximately 60 GB of available device memory.

\textbf{Models and Metrics.} Our experiments use Qwen2.5-32B \citep{qwen2025qwen25technicalreport} for dense architectures and a around 500B open-source MoE model for MoE architectures. We measure \textit{end-to-end throughput} (tokens/s), calculated as the total tokens (prompts and responses) processed per second, and \textit{average step time} (min), which represents the wall-clock time per RL iteration. All reported numbers are averaged over five RL iterations after the warm-up phase to reflect steady-state performance.

\textbf{Datasets.} We utilize the "DAPO-Math-17k" dataset for training, with the maximum input and output sequence lengths set to 2K and 30K tokens, respectively.

\textbf{Baselines and Implementation.} We compare DORA against three representative RL training paradigms: (1)~\textit{Synchronous (All-Colocated)}, which satisfies all constraints but suffers from batch barriers; (2)~\textit{One-step off-policy}, which satisfies all constraints with staleness $K{=}1$ but only overlaps pipeline stages without eliminating rollout bubbles; and (3)~\textit{Partial rollout} in All-colocated implementation similar to the work~\citep{team2025kimi}. All baselines are implemented within the same in-house RL framework to ensure a controlled comparison under identical hardware and software configurations. Our RL system uses vLLM~\citep{kwon2023efficient} as the inference engine, Megatron-LM~\citep{megatron-lm} as the training backend, and extends torch RPC~\citep{damania2023pytorch} with streaming primitives. The software environment includes CUDA-12.4, PyTorch-2.6.0, vLLM-0.8.5, and NCCL-2.28.

\textbf{Training Configurations.} For the RL algorithm, we follow the setting of DAPO \citep{yu2025dapo}, a variant of GRPO. Each rollout consists of a prompt batch size of 512, with 16 responses sampled per prompt, resulting in a global training batch size of 8,192. Each training iteration involves 16 update steps with a micro-batch size of 512. To stress-test efficiency under realistic long-tailed generation, we select an intermediate checkpoint where the mean response length is 2.4K tokens and the maximum reaches 30K.

\section{Case Study: Empirical Validation of Multi-Version Streaming}
\label{app:case-study}
This part presents a detailed empirical case study that validates the algorithmic properties of DORA's multi-version streaming training. Using rollout logs spanning 22{,}820 records (365{,}120 responses) over training steps 20--100 with staleness bound $K{=}3$, we examine three modes: \emph{DORA}, \emph{Partial Rollout}, and \emph{Synchronous} training. Each record contains 16 responses scored by a binary reward model ($+1$/$-1$).

\subsection{Observation 1: Bounded Staleness Incurs Nearly Zero Quality Degradation}
\label{app:staleness-zero}

A natural concern with asynchronous training is that stale data degrades learning. We perform a controlled analysis to disentangle staleness from confounding variables.

\textbf{Na\"ive observation.} Aggregating across all records, staleness-0 data achieves a 58.9\% pass rate versus 55.6\% for staleness-1, an apparent 3.3 percentage point gap.

\textbf{Controlled analysis.} This gap is entirely explained by \emph{selection bias}: harder problems produce longer responses, which take more time to generate and thus receive higher staleness labels. Table~\ref{tab:staleness-controlled} stratifies records by difficulty (pass rate bin) and compares staleness-0 versus staleness-1 \emph{within each stratum}:

\begin{table}[h]
\centering
\caption{Pass rate (\%) by staleness, stratified by problem difficulty (DORA, $K{=}3$). The staleness gap vanishes after controlling for difficulty, confirming zero quality degradation.}
\label{tab:staleness-controlled}
\small
\begin{tabular}{lccccc}
\toprule
\textbf{Difficulty Bin} & \textbf{s=0 (\%)} & \textbf{n (s=0)} & \textbf{s=1 (\%)} & \textbf{n (s=1)} & \textbf{$\Delta$ (\%)} \\
\midrule
Hard (PR $<$ 0.25) & 10.93 & 187 & 11.27 & 667 & $+$0.34 \\
Medium (0.25--0.5) & 34.42 & 146 & 34.39 & 580 & $-$0.03 \\
Medium-Easy (0.5--0.75) & 60.67 & 184 & 60.05 & 625 & $-$0.62 \\
Easy (PR $>$ 0.75) & 86.55 & 441 & 86.55 & 1263 & $+$0.00 \\
\bottomrule
\end{tabular}
\end{table}

We further verify this at the per-response level using Partial Rollout's generation segments (Table~\ref{tab:staleness-per-resp}), which provides 141{,}632 responses with explicit version labels:

\begin{table}[h]
\centering
\caption{Per-response staleness analysis (Partial Rollout, $n{=}141{,}632$). Within each difficulty bin, staleness-0 and staleness-1 responses perform identically.}
\label{tab:staleness-per-resp}
\small
\begin{tabular}{lcccc}
\toprule
\textbf{Pass Rate Bin} & \textbf{s=0 (\%)} & \textbf{s=1 (\%)} & \textbf{$\Delta$ (\%)} & \textbf{n (total)} \\
\midrule
$[0, 0.125)$ & 6.25 & 6.25 & 0.00 & 14{,}384 \\
$[0.125, 0.25)$ & 15.43 & 15.28 & $-$0.15 & 17{,}776 \\
$[0.25, 0.5)$ & 34.04 & 33.86 & $-$0.17 & 25{,}712 \\
$[0.5, 0.75)$ & 60.18 & 59.85 & $-$0.33 & 28{,}880 \\
$[0.75, 0.875)$ & 78.37 & 78.63 & $+$0.26 & 19{,}344 \\
$[0.875, 1.0]$ & 91.37 & 91.28 & $-$0.10 & 35{,}520 \\
\bottomrule
\end{tabular}
\end{table}

\textbf{Mechanism.} Figure~\ref{fig:case_staleness}(b) confirms the causal pathway: staleness-1 responses have a median token count of 1{,}495 versus 1{,}015 for staleness-0 (47\% longer). Longer generation implies harder problems; under DORA's asynchronous scheduling, these naturally complete after a weight update and receive a higher staleness label. The \emph{on-policy} improvement rate is only 0.113\%/step (measured from on-policy data over steps 50--100), so the theoretical maximum penalty for $K{=}3$ is merely $3 \times 0.113 = 0.34\%$, well below sampling noise. This validates that DORA's sliding-window staleness control (Section~\ref{sec:Multi-version Asynchronous Training}) is sufficient for convergence parity.

\begin{figure}[t]
  \centering
  \includegraphics[width=0.95\textwidth]{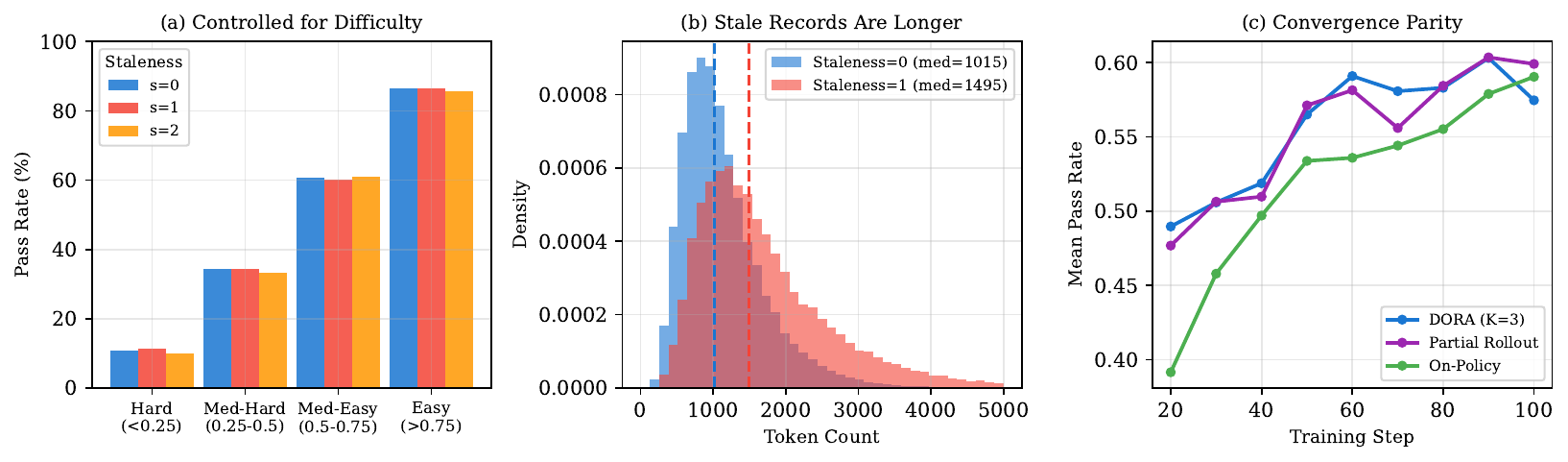}
  \caption{(a)~After controlling for problem difficulty, the staleness gap vanishes ($\Delta < 0.6\%$ across all bins). (b)~Selection bias mechanism: stale records are systematically longer (harder). (c)~All three paradigms converge to the same pass-rate range (0.57--0.60).}
  \label{fig:case_staleness}
\end{figure}

\subsection{Observation 2: Long-Tail Trajectories Carry the Highest Training Signal}
\label{app:longtail-value}

DORA's trajectory-level streaming preserves every sampled trajectory (Section~\ref{sec:Multi-version Asynchronous Training}), in contrast to replication-based methods that discard long-tail samples once enough short ones complete. We validate the importance of this design by measuring the GRPO gradient signal, quantified as within-group reward variance, across difficulty levels:

\begin{table}[h]
\centering
\caption{Reward variance (gradient signal strength for GRPO) by difficulty. Medium-difficulty problems with the highest variance are also the longest and most likely to be stale, i.e., exactly the trajectories that replication-based methods discard first.}
\label{tab:gradient-signal}
\small
\begin{tabular}{lcccc}
\toprule
\textbf{Difficulty} & \textbf{n} & \textbf{Reward Var.} & \textbf{Avg Length} & \textbf{Avg Staleness} \\
\midrule
Hard (1--3/16 correct) & 1{,}007 & 0.409 & 3{,}755 & 0.99 \\
Medium-Hard (4--7/16) & 868 & \textbf{0.939} & 3{,}624 & 1.01 \\
Medium-Easy (8--11/16) & 950 & \textbf{1.001} & 3{,}418 & 0.96 \\
Easy (12--14/16) & 1{,}206 & 0.617 & 3{,}025 & 0.90 \\
Very Easy (15--16/16) & 721 & 0.250 & 2{,}711 & 0.87 \\
\bottomrule
\end{tabular}
\end{table}

The medium-difficulty records (pass rate 0.25--0.75) exhibit the highest reward variance ($\sim$1.0), providing the strongest gradient signal. These records are simultaneously the longest (3{,}400--3{,}800 chars) and the most stale (staleness 0.96--1.01). In a replication-based system that terminates generation once $TBS$ short trajectories are collected, these high-signal samples would be the first to be discarded. DORA preserves all of them via multi-version streaming, ensuring the integrity of advantage estimation in GRPO.

\begin{figure}[t]
  \centering
  \includegraphics[width=0.85\textwidth]{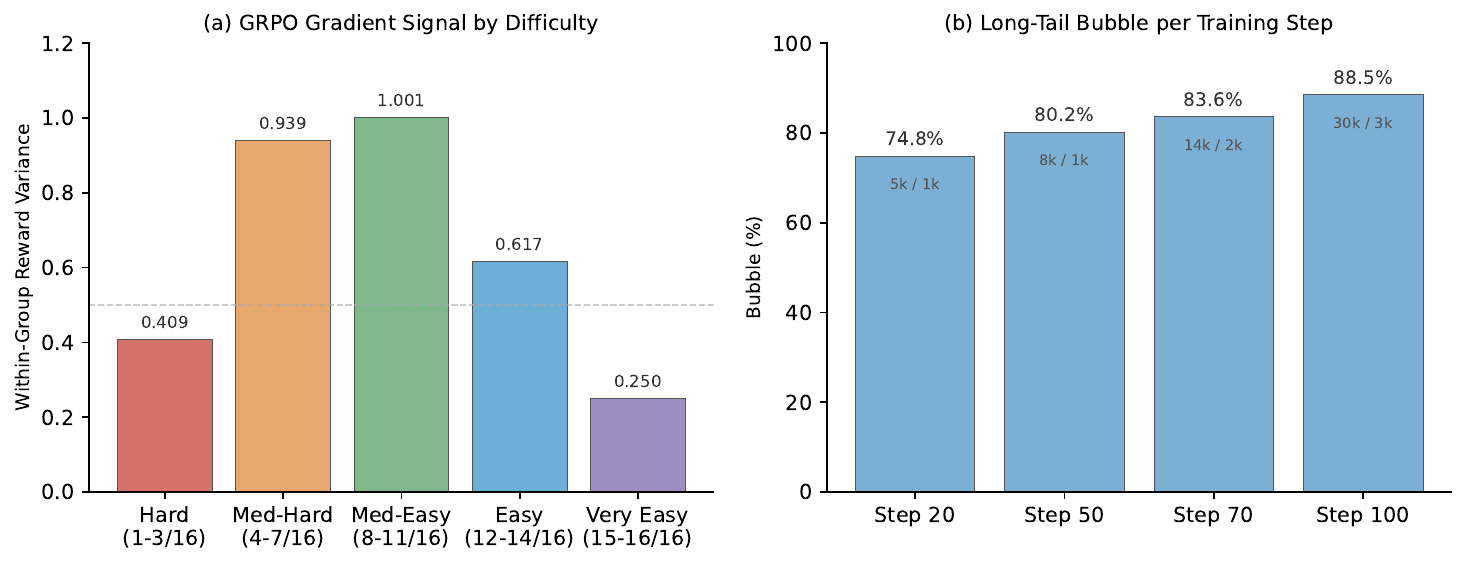}
  \caption{(a)~GRPO gradient signal peaks at medium difficulty (pass rate 0.25--0.75). (b)~Under synchronous training, the inter-node bubble exceeds 75\% across all steps and grows to 88.5\% at step 100 as responses lengthen.}
  \label{fig:case_longtail}
\end{figure}

\subsection{Observation 3: Quantifying the Long-Tail Bubble}
\label{app:bubble}

We directly measure the inter-node bubble that DORA could theoretically eliminate from the view of the request workloads. For each training step, we compute the fraction of decode time wasted waiting for the longest trajectory:
$$\text{Bubble} = \frac{\max_i L_i - \overline{L_i}}{\max_i L_i}$$
As shown in Figure~\ref{fig:case_longtail}(b), the bubble exceeds 74\% at step 20 and grows to \textbf{88.5\%} at step 100 (max response: 30{,}720 tokens \textit{vs.}\ mean: 3{,}536). The max-to-median token ratio reaches \textbf{28$\times$}, consistent with the order-of-magnitude skew reported in Section~\ref{sec:longtail}. This quantifies the efficiency opportunity that DORA captures through multi-version streaming: by allowing the longest trajectories to continue under their legacy version while new requests run on the latest policy model, DORA could theoretically eliminate this 75--89\% bubble entirely from the viewpoint of engineering only.

\subsection{Summary}

This case study empirically characterizes DORA's multi-version streaming training along three axes. Two concern its algorithmic design principles, and one quantifies the efficiency opportunity that motivates the design.
\begin{itemize}[leftmargin=1.2em, itemsep=2pt]
    \item \textbf{Bounded staleness (design principle).} After controlling for problem difficulty, staleness $\leq 3$ produces nearly zero measurable quality degradation ($\Delta < 0.6\%$ across all strata); the apparent staleness gap is entirely explained by selection bias.
    \item \textbf{Preserving every sampled trajectory (design principle).} The longest, most stale trajectories carry the highest GRPO gradient signal (reward variance $\sim$1.0). Discarding them, as replication-based methods do, removes the most valuable training data.
    \item \textbf{Quantifying the long-tail bubble (motivation).} Under synchronous training the inter-node bubble reaches 75--89\%, confirming that the inefficiency DORA targets is substantial.
\end{itemize}
Together, these observations support DORA's design goal of resolving the long-tail dilemma without introducing algorithmic corrections.

\section{A Policy-Improvement View of Single-Version Rollout in DORA vs.\ Partial Rollout}
\label{sec:policy-improvement}
This appendix examines the algorithmic dimension of DORA's single-policy-per-trajectory
design through its effect on the \emph{policy-improvement guarantee} underlying PPO-style
updates; a complementary variance-based analysis follows in
Appendix~\ref{sec:baseline-view}. In asynchronous training a batch mixes trajectories from
multiple policy versions, and DORA and partial-rollout
methods~\citep{team2025kimi, wu2025llamarl, du2025ulorl, fu2025areal} differ in how such a
batch enters this guarantee. We examine this through the classical trust-region /
sample-reuse framework for monotonic policy
improvement~\citep{schulman2015trust, achiam2017constrained, queeney2021generalized}. Our
goal is modest: to locate DORA and partial rollout as two regimes of a common
mixed-behavior formulation, and to identify the structural consequence---absence versus
presence of an advantage-substitution bias---that distinguishes them.

\paragraph{Scope and abstraction.}
The analysis below is a \emph{stylized comparison under a common formulation}, not a
faithful model of either system's full training dynamics. We adopt the
\emph{mixed-behavior-policy formulation}, in which the policy version index is lifted into
the state~\citep{espeholt2018impala}: a batch drawn from lagged versions $\{\pi_{w_k}\}$ is
described by an augmented state $(s, i) \in \mathcal{S} \times \mathcal{I}$, a mixed behavior
policy $\beta(a \mid s, i) := \pi_{w_i}(a \mid s)$, and an \emph{index-transition kernel}
$q(i' \mid i)$ that governs how the active version evolves along a trajectory. Within this
formulation, the two asynchronous paradigms are the two representative values of $q$:
\begin{itemize}
\item \textbf{Trajectory-level mixing (DORA).} Each trajectory is generated end-to-end under
the version active at its dispatch, so the index never changes:
$q(i' \mid i) = \mathbf{1}\{i' = i\}$.
\item \textbf{Step/segment-level mixing (partial rollout).} A trajectory may switch to a
newer version at a weight update, with switching probability $\sigma > 0$:
$q(i' \mid i) = (1-\sigma)\,\mathbf{1}\{i' = i\} + \sigma\,\kappa(i' \mid i)$, where $\kappa$
is the target-version distribution.
\end{itemize}
This formulation is faithful to DORA's actual generation process, which tags each prompt
with a single version upon dispatch (Section~\ref{sec:methods}). It uses $\sigma > 0$ as a
stylized representative of the step/segment-level family; it does not model the additional
correction mechanisms that concrete partial-rollout systems apply to their stitched
trajectories---decoupled objectives, importance corrections across segments, gradient
masking, and so on---which lie outside our scope. Throughout, $\pi_\theta$ denotes the
target policy, $\{\pi_{w_k}\}_{k=1}^{K}$ the lagged policies with schedule $\{p_k\}$, and
$\mu_k = J(\pi_{w_k})$ their expected rewards.

\subsection{Setup and Notation}
\label{subsec:pi-setup}

The phenomenon studied here---mid-trajectory version switching---is intrinsically
token-level: a weight update can change the active policy version between two tokens of the
same response. We therefore adopt a token-level MDP in which a response is a length-$H$
trajectory $(a_1, \dots, a_H)$ and the active version may evolve along it. Let
$J(\pi) = \mathbb{E}[\sum_{t\ge0} \gamma^t r(s_t, a_t) \mid \pi]$, let $d_\pi$ denote the
discounted state-visitation distribution, and let $A^{\pi}$ denote the advantage function.
For a behavior policy $\pi_{w}$, define the importance-weighted surrogate objective
\begin{equation}
L_{\pi_{w}}(\pi_\theta)
:= \frac{1}{1-\gamma}\,
\mathbb{E}_{s \sim d_{\pi_{w}},\, a \sim \pi_{w}(\cdot\mid s)}
\!\left[\frac{\pi_\theta(a \mid s)}{\pi_{w}(a \mid s)}\, A^{\pi_{w}}(s,a)\right],
\label{eq:surrogate}
\end{equation}
and the expected-advantage coefficient
$C_{\pi_\theta, \pi_{w}} := \max_{s} \bigl| \mathbb{E}_{a \sim \pi_\theta}[A^{\pi_{w}}(s,a)] \bigr|$.
On the augmented state space, the mixed behavior policy $\beta$ induces a visitation
distribution $d_\beta(s,i)$ and an advantage $A^{\beta}((s,i),a)$; the index-transition
kernel $q$ enters only through the future evolution of $i$. We write
$J_{\mathrm{mix}} := J(\beta)$ for the return of the mixed behavior policy, and
$C_{\pi_\theta,\beta} := \max_{(s,i)} | \mathbb{E}_{a\sim\pi_\theta}[A^{\beta}((s,i),a)] |$.

\subsection{Monotonic Improvement under Trajectory-Level Mixing}
\label{subsec:pi-theorem}

Under trajectory-level mixing the augmented-state advantage collapses to the per-version
advantage, which lets the standard sample-reuse bound apply verbatim.

\begin{lemma}[Advantage reduction under trajectory-level mixing]
\label{lem:adv-reduction}
If $q(i' \mid i) = \mathbf{1}\{i' = i\}$, then $d_\beta(s,i) = \alpha_i\, d_{\pi_{w_i}}(s)$
and $A^{\beta}((s,i),a) = A^{\pi_{w_i}}(s,a)$ for every $(s,i)$, where $\alpha_i$ is the batch
fraction of version $i$. Consequently $J_{\mathrm{mix}} = \sum_i \alpha_i J(\pi_{w_i})$.
\end{lemma}
Since the index never changes, every future rollout from $(s,i)$ is generated by the same
$\pi_{w_i}$, so the augmented-state value and advantage reduce to those of $\pi_{w_i}$; a
proof is given in Appendix~\ref{subsec:proof-lemma}.

\begin{theorem}[Monotonic improvement, trajectory-level mixing]
\label{thm:mono-improve}
Let the batch be a trajectory-level mixture ($q = \mathbf{1}\{i'=i\}$) with version fractions
$\{\alpha_i\}$. Then, under the standard trust-region regularity conditions
of~\citet{schulman2015trust, achiam2017constrained},
\begin{equation}
J(\pi_\theta) - \sum_i \alpha_i\, J(\pi_{w_i})
\;\ge\;
\sum_i \alpha_i\, L_{\pi_{w_i}}(\pi_\theta)
\;-\;
\frac{2\gamma \max_i C_{\pi_\theta, \pi_{w_i}}}{(1-\gamma)^2}
\sum_i \alpha_i\,
\mathbb{E}_{s \sim d_{\pi_{w_i}}}\!\bigl[D_{\mathrm{TV}}(\pi_\theta, \pi_{w_i}; s)\bigr].
\label{eq:mono-improve}
\end{equation}
\end{theorem}
Equation~\eqref{eq:mono-improve} adapts the multi-version sample-reuse improvement bound
of~\citet{queeney2021generalized} to the mixed-behavior setting. Unlike that bound, which
anchors the advantage at the current policy and measures improvement relative to it, our
bound anchors at the per-version advantage $A^{\pi_{w_i}}$ and measures improvement relative
to the mixture baseline $\sum_i \alpha_i J(\pi_{w_i})$; the proof (via the
performance-difference lemma~\citep{kakade2002approximately} and the average-TV bound
of~\citet{achiam2017constrained}) is given in Appendix~\ref{subsec:proof-theorem}. The key
structural point, supplied by Lemma~\ref{lem:adv-reduction}, is that the surrogate uses each
trajectory's \emph{own} behavior version $\pi_{w_i}$ together with its \emph{own} advantage
$A^{\pi_{w_i}}$. Consequently the per-trajectory importance ratio $\pi_\theta / \pi_{w_i}$---
exactly the ratio $r_{i,t}$ in DORA's objective---is a well-defined
single-policy ratio, and the bound holds \emph{without any cross-version correction}.

\subsection{Advantage-Substitution Bias under Step/Segment-Level Mixing}
\label{subsec:pi-corollary}

The reduction of Lemma~\ref{lem:adv-reduction} is exactly what fails once trajectories
switch versions mid-generation. This is the additive observation of this appendix.

\begin{corollary}[Single-version generation removes advantage-substitution bias]
\label{cor:adv-bias}
Under step/segment-level mixing ($\sigma > 0$), the augmented-state advantage
$A^{\beta}((s,i),a)$ depends on the future index evolution induced by $q$ and does not in
general equal any single-version advantage $A^{\pi_{w_i}}$. Estimating the surrogate with
$A^{\pi_{w_i}}$ in place of $A^{\beta}$ therefore incurs an advantage-substitution bias
\begin{equation}
\varepsilon_{\mathrm{sub}}(i)
:= \Big|\, \mathbb{E}_{s\sim d_{\pi_{w_i}},\, a\sim\pi_{w_i}(\cdot\mid s)}
\big[A^{\beta}((s,i),a) - A^{\pi_{w_i}}(s,a)\big] \,\Big|,
\label{eq:adv-bias}
\end{equation}
which is nonzero for $\sigma > 0$ (whenever the switched-to version differs) and vanishes as
$\sigma \to 0$. This bias enters the improvement guarantee directly. Writing
$\hat{L}(\pi_\theta)$ for the surrogate actually optimized with the single-version advantage
$A^{\pi_{w_i}}$, and $L^{\beta}(\pi_\theta) := \frac{1}{1-\gamma}\mathbb{E}_{(s,i)\sim d_\beta,\,
a\sim\pi_{w_i}}[\frac{\pi_\theta(a\mid s)}{\pi_{w_i}(a\mid s)}A^{\beta}((s,i),a)]$ for the
surrogate required by the augmented-state bound, a bounded importance ratio
$\pi_\theta/\pi_{w_i} \le M_\rho$ gives
$|L^{\beta}(\pi_\theta) - \hat{L}(\pi_\theta)| \le \frac{M_\rho}{1-\gamma}\,
\mathbb{E}_i[\varepsilon_{\mathrm{sub}}(i)]$. Substituting into the augmented-state improvement
bound yields, for any $\sigma$,
\begin{equation}
\begin{aligned}
J(\pi_\theta) - J_{\mathrm{mix}}
\;\ge\;\;
& \hat{L}(\pi_\theta)
\;-\;
\frac{2\gamma\, C_{\pi_\theta,\beta}}{(1-\gamma)^2}\,
\mathbb{E}_{(s,i)\sim d_\beta}\!\big[D_{\mathrm{TV}}(\pi_\theta, \pi_{w_i}; s)\big] \\
& \;-\;
\underbrace{\frac{M_\rho}{1-\gamma}\, \mathbb{E}_i[\varepsilon_{\mathrm{sub}}(i)]}_{\text{substitution penalty}}.
\end{aligned}
\label{eq:mono-improve-bias}
\end{equation}
Trajectory-level mixing ($\sigma = 0$, DORA) is the unique regime in which
$\varepsilon_{\mathrm{sub}} \equiv 0$: the substitution penalty vanishes, the reduction of
Lemma~\ref{lem:adv-reduction} holds exactly, and the standard GRPO objective
directly realizes the bound's surrogate without bias.
\end{corollary}

\paragraph{$\hat{L}$ is a best-case proxy.}
We take $\hat{L}$ to use the clean single-version advantage $A^{\pi_{w_i}}$; this is a
\emph{best-case} proxy for partial rollout, which in practice cannot even access
$A^{\pi_{w_i}}$ exactly---its critic is trained on version-mixed data, or its group-relative
advantage is shared across a version-switched trajectory. The gap $\varepsilon_{\mathrm{sub}}$
is therefore a \emph{lower bound} on the true substitution error; even in this idealized case
it is nonzero for $\sigma>0$, whereas it vanishes identically for DORA ($\sigma=0$).

\subsection{Discussion: Consequences for DORA}
\label{subsec:pi-discussion}

Theorem~\ref{thm:mono-improve} and Corollary~\ref{cor:adv-bias} together give a
policy-improvement reading of DORA's single-policy-per-trajectory design along three axes,
each mapping to a claim in the main text.

\paragraph{No algorithmic correction (Section~\ref{sec:intro}, ``without algorithmic compromises'').}
By Corollary~\ref{cor:adv-bias}, DORA's $\sigma = 0$ makes the advantage reduction exact, so
its surrogate is realized by the unmodified GRPO objective. Partial
rollout's $\sigma > 0$ breaks the reduction, which is why concrete systems must compensate
with decoupled objectives~\citep{fu2025areal} or gradient masking~\citep{team2025kimi}.
DORA's freedom from such corrections is thus a direct consequence of trajectory-level mixing,
not an engineering simplification.

\paragraph{Well-defined importance ratio (Section~\ref{sec:methods}).}
Because each DORA trajectory corresponds to a single behavior version $\pi_{w_i}$, its
sequence-level ratio $\prod_t \pi_\theta(a_t\mid\cdot)/\pi_{w_i}(a_t\mid\cdot)$ is a
well-defined single-policy ratio. A trajectory stitched from multiple versions corresponds to
no single behavior policy, so its sequence-level ratio loses this interpretation at segment
boundaries---the algorithmic root of the corrections above.

\paragraph{Staleness cost and its control (Section~\ref{sec:methods}, staleness bound $K$).}
The total-variation term in~\eqref{eq:mono-improve} grows with the deviation
$D_{\mathrm{TV}}(\pi_\theta, \pi_{w_i};s)$ between the target policy and each behavior version.
DORA controls this term on the sampling side: the sliding window enforces a deterministic
staleness bound $v(\theta) - v(w_i) \le K$ (Section~\ref{sec:methods}), so the summands with
the largest deviation are excluded by construction. This is the improvement-side counterpart
of the convergence--throughput role of $K$ discussed in the main text and its limitation
(manual configuration, reliance on clipping). We note the trade-off honestly: step/segment-level
mixing can reduce this deviation faster, since switching continually refreshes trajectories
toward the latest version, whereas DORA holds long-tail trajectories at their dispatch version
and instead bounds the deviation through $K$. The two paradigms are therefore not uniformly
ordered; DORA trades a higher per-trajectory staleness ceiling---bounded by $K$---for an exact
advantage reduction (Corollary~\ref{cor:adv-bias}), a well-defined single-policy ratio, and the
KV-Cache equivalence that enables zero-re-prefill migration (Section~\ref{sec:methods}).

\section{Proofs for the Policy-Improvement Analysis}
\label{sec:pi-proofs}
This appendix collects the proofs for the results of Appendix~\ref{sec:policy-improvement}.
All arguments are carried out on the \emph{augmented} MDP, so that the mixed behavior policy
is treated as a single Markov policy throughout; the per-version form (for DORA) is then
recovered through Lemma~\ref{lem:adv-reduction}. We first record the regularity conditions
and the performance-difference identity (Appendix~\ref{subsec:pi-regularity}), then prove
Lemma~\ref{lem:adv-reduction} (Appendix~\ref{subsec:proof-lemma}),
Theorem~\ref{thm:mono-improve} (Appendix~\ref{subsec:proof-theorem}), and
Corollary~\ref{cor:adv-bias} (Appendix~\ref{subsec:proof-corollary}).

\subsection{Regularity Conditions and the Performance-Difference Identity}
\label{subsec:pi-regularity}

We work in the token-level MDP of Appendix~\ref{subsec:pi-setup}, with augmented state space
$\tilde{\mathcal{S}} = \mathcal{S}\times\mathcal{I}$, mixed behavior policy
$\beta(a\mid s,i) := \pi_{w_i}(a\mid s)$, and index-transition kernel $q(i'\mid i)$. The
augmented MDP inherits the reward and environment transition of the base MDP, while the index
evolves according to $q$ independently of the action. Its initial state is drawn as
$(s_0, i_0) \sim \rho_0 \times \alpha$, and the target policy $\pi_\theta$ is lifted to the
augmented space by ignoring the index, $\pi_\theta(a\mid s,i) := \pi_\theta(a\mid s)$. Since
the augmented transition is the product of the base transition and the index kernel $q$, it is
a valid Markov kernel; consequently the state-distribution and performance-difference results
below, established for general MDPs by Kakade and Langford~\citep{kakade2002approximately} and
Achiam et al.~\citep{achiam2017constrained}, apply on the augmented MDP verbatim.

\paragraph{Horizon convention.} An LLM response is a finite length-$H$ trajectory, whereas the
bounds below are stated in the infinite-horizon discounted form for continuity with the
classical policy-improvement literature~\citep{kakade2002approximately, schulman2015trust,
achiam2017constrained}. The same decomposition holds in the finite-horizon setting, with the
factor $1/(1-\gamma)$ replaced by a horizon-dependent constant; we adopt the discounted
notation throughout and do not track this substitution.

We assume:
\begin{itemize}
\item \textbf{Discounting.} $\gamma \in (0,1)$, so every Bellman operator below is a
contraction and the value functions are well defined.
\item \textbf{Bounded advantage.} There exists $A_{\max} < \infty$ with
$|A^{\pi}(s,a)| \le A_{\max}$ for the policies considered.
\item \textbf{Common support.} For every behavior version $\pi_{w_i}$ entering the objective,
$\pi_\theta(a\mid s) > 0 \Rightarrow \pi_{w_i}(a\mid s) > 0$, so the ratio
$\pi_\theta/\pi_{w_i}$ is well defined.
\item \textbf{Bounded ratio.} $\pi_\theta(a\mid s)/\pi_{w_i}(a\mid s) \le M_\rho$, as enforced
in practice by the clipping mechanism of PPO/GRPO~\citep{schulman2017proximal, shao2024deepseekmath}.
\end{itemize}

We use the performance-difference lemma~\citep{kakade2002approximately} on the augmented MDP:
for any augmented-space policies $\pi, \pi'$,
\begin{equation}
J(\pi) - J(\pi') = \frac{1}{1-\gamma}\,
\mathbb{E}_{(s,i)\sim d_{\pi}}\!\Big[\mathbb{E}_{a\sim\pi(\cdot\mid s,i)}\big[A^{\pi'}((s,i),a)\big]\Big],
\label{eq:pdl}
\end{equation}
and its consequence, the single-behavior trust-region improvement
bound~\citep{schulman2015trust, achiam2017constrained}: for a behavior policy $\mu$ whose
support covers $\pi$,
\begin{equation}
J(\pi) - J(\mu) \;\ge\;
\frac{1}{1-\gamma}\,\mathbb{E}_{(s,i)\sim d_{\mu},\, a\sim\mu(\cdot\mid s,i)}\!\Big[\tfrac{\pi(a\mid s,i)}{\mu(a\mid s,i)}A^{\mu}((s,i),a)\Big]
\;-\;
\frac{2\gamma\, C_{\pi,\mu}}{(1-\gamma)^2}\,
\mathbb{E}_{(s,i)\sim d_{\mu}}\!\big[D_{\mathrm{TV}}(\pi,\mu;(s,i))\big],
\label{eq:tr-bound}
\end{equation}
with $C_{\pi,\mu} := \max_{(s,i)}\big|\mathbb{E}_{a\sim\pi(\cdot\mid s,i)}[A^{\mu}((s,i),a)]\big|$.
The bound~\eqref{eq:tr-bound} follows from~\eqref{eq:pdl} by replacing the on-policy visitation
$d_\pi$ with the behavior visitation $d_\mu$ and controlling the induced distribution mismatch
via the average total-variation bound of Achiam et al.~\citep{achiam2017constrained}, which
holds on the augmented MDP by the remark above.

\subsection{Proof of Lemma~\ref{lem:adv-reduction} (Advantage Reduction)}
\label{subsec:proof-lemma}

\begin{proof}
Assume trajectory-level mixing, $q(i'\mid i) = \mathbf{1}\{i'=i\}$. Under this kernel the index
component is invariant: starting from $(s,i)$, every subsequent augmented state has the form
$(s',i)$ with the same $i$. Hence along any trajectory initialized at $(s,i)$ the behavior
policy $\beta(\cdot\mid\cdot,i) = \pi_{w_i}(\cdot\mid\cdot)$ acts as the fixed base-MDP policy
$\pi_{w_i}$, and the augmented transition reduces to the base transition under $\pi_{w_i}$.
Since $\gamma<1$, the augmented Bellman equation has a unique solution, which therefore
coincides with the base-MDP value of $\pi_{w_i}$:
\begin{equation*}
V^{\beta}(s,i)
= \mathbb{E}\Big[\textstyle\sum_{t\ge0}\gamma^t r(s_t,a_t)\,\Big|\, s_0=s,\ \text{index frozen at } i,\ a_t\sim\pi_{w_i}\Big]
= V^{\pi_{w_i}}(s),
\end{equation*}
and likewise $Q^{\beta}((s,i),a) = Q^{\pi_{w_i}}(s,a)$. Subtracting gives
$A^{\beta}((s,i),a) = A^{\pi_{w_i}}(s,a)$ for every $(s,i)$; this is a property of the value
functions and does not depend on the visitation distribution. Finally, because the index is
drawn once as $i_0\sim\alpha$ and then held fixed while the base state evolves under $\pi_{w_i}$
from $s_0\sim\rho_0$, the augmented visitation factorizes as
$d_\beta(s,i) = \alpha_i\, d_{\pi_{w_i}}(s)$, where $d_{\pi_{w_i}}$ is the base-MDP visitation of
$\pi_{w_i}$ started from $\rho_0$. Integrating the reward identity over $d_\beta$ then gives
$J_{\mathrm{mix}} = J(\beta) = \sum_i \alpha_i J(\pi_{w_i})$.
\end{proof}

\subsection{Proof of Theorem~\ref{thm:mono-improve} (Monotonic Improvement)}
\label{subsec:proof-theorem}

\begin{proof}
We apply the trust-region bound~\eqref{eq:tr-bound} on the augmented MDP with target
$\pi = \pi_\theta$ (lifted) and behavior $\mu = \beta$; this step uses only that $\beta$ is a
Markov policy on the augmented MDP and holds for any kernel $q$. Since $\pi_\theta$ ignores the
index and $\beta(\cdot\mid s,i) = \pi_{w_i}(\cdot\mid s)$, the two policies at a fixed $(s,i)$
are $\pi_\theta(\cdot\mid s)$ and $\pi_{w_i}(\cdot\mid s)$, so both the ratio and the total
variation reduce to base-MDP quantities,
\begin{equation*}
\tfrac{\pi_\theta(a\mid s,i)}{\beta(a\mid s,i)} = \tfrac{\pi_\theta(a\mid s)}{\pi_{w_i}(a\mid s)},
\qquad
D_{\mathrm{TV}}(\pi_\theta,\beta;(s,i)) = D_{\mathrm{TV}}(\pi_\theta,\pi_{w_i};s).
\end{equation*}
With $J(\beta) = J_{\mathrm{mix}}$, \eqref{eq:tr-bound} becomes the augmented-state bound
\begin{equation}
J(\pi_\theta) - J_{\mathrm{mix}}
\;\ge\;
\frac{1}{1-\gamma}\,\mathbb{E}_{(s,i)\sim d_{\beta},\, a\sim\pi_{w_i}}\!\Big[\tfrac{\pi_\theta(a\mid s)}{\pi_{w_i}(a\mid s)}A^{\beta}((s,i),a)\Big]
\;-\;
\frac{2\gamma\, C_{\pi_\theta,\beta}}{(1-\gamma)^2}\,
\mathbb{E}_{(s,i)\sim d_{\beta}}\!\big[D_{\mathrm{TV}}(\pi_\theta,\pi_{w_i};s)\big],
\label{eq:aug-bound}
\end{equation}
which holds for any kernel $q$. We now specialize to trajectory-level mixing and expand via
Lemma~\ref{lem:adv-reduction}. The factorization $d_\beta(s,i) = \alpha_i\, d_{\pi_{w_i}}(s)$
turns every augmented expectation $\mathbb{E}_{(s,i)\sim d_\beta}[\,\cdot\,]$ into
$\sum_i \alpha_i\, \mathbb{E}_{s\sim d_{\pi_{w_i}}}[\,\cdot\,]$, and the reduction
$A^{\beta}((s,i),a) = A^{\pi_{w_i}}(s,a)$ replaces the augmented advantage by the per-version
advantage. The surrogate term becomes
\begin{equation*}
\frac{1}{1-\gamma}\sum_i \alpha_i\,
\mathbb{E}_{s\sim d_{\pi_{w_i}},\, a\sim\pi_{w_i}}\!\Big[\tfrac{\pi_\theta(a\mid s)}{\pi_{w_i}(a\mid s)}A^{\pi_{w_i}}(s,a)\Big]
= \sum_i \alpha_i\, L_{\pi_{w_i}}(\pi_\theta),
\end{equation*}
with $L_{\pi_{w_i}}$ as in~\eqref{eq:surrogate}, and the penalty term becomes
$\frac{2\gamma C_{\pi_\theta,\beta}}{(1-\gamma)^2}\sum_i \alpha_i\,
\mathbb{E}_{s\sim d_{\pi_{w_i}}}[D_{\mathrm{TV}}(\pi_\theta,\pi_{w_i};s)]$. Finally, since the
maximum over the product space factorizes and $A^{\beta}((s,i),a)=A^{\pi_{w_i}}(s,a)$,
\begin{equation*}
C_{\pi_\theta,\beta}
= \max_{(s,i)}\big|\mathbb{E}_{a\sim\pi_\theta}[A^{\beta}((s,i),a)]\big|
= \max_i \max_{s}\big|\mathbb{E}_{a\sim\pi_\theta}[A^{\pi_{w_i}}(s,a)]\big|
= \max_i C_{\pi_\theta,\pi_{w_i}}.
\end{equation*}
Substituting these three identities into~\eqref{eq:aug-bound} yields~\eqref{eq:mono-improve}.
\end{proof}

\subsection{Proof of Corollary~\ref{cor:adv-bias} (Advantage-Substitution Bias)}
\label{subsec:proof-corollary}

Throughout this proof we work under the stylized kernel
$q(i'\mid i) = (1-\sigma)\mathbf{1}\{i'=i\} + \sigma\,\kappa(i'\mid i)$ of
Appendix~\ref{subsec:pi-setup}, and $A^{\pi_{w_i}}$ denotes the advantage of the current-segment
version, i.e., the version indexed by the augmented state $(s,i)$.

\begin{proof}
When $\sigma = 0$ the kernel is the identity transition and
Lemma~\ref{lem:adv-reduction} gives $A^{\beta}((s,i),a) = A^{\pi_{w_i}}(s,a)$ pointwise, hence
$\varepsilon_{\mathrm{sub}}(i) = 0$. When $\sigma > 0$, the augmented value obeys
\begin{equation}
V^{\beta}(s,i)
= \mathbb{E}_{a\sim\pi_{w_i}}\Big[ r(s,a) + \gamma\,
\mathbb{E}_{s'}\big[ (1-\sigma) V^{\beta}(s',i) + \sigma\,
\mathbb{E}_{i'\sim\kappa(\cdot\mid i)} V^{\beta}(s',i') \big] \Big],
\label{eq:beta-bellman}
\end{equation}
which differs from the $\sigma=0$ recursion (whose solution is $V^{\pi_{w_i}}$) through the
$\sigma$-weighted term: with probability $\sigma$ the continuation is evaluated under a version
$i'\sim\kappa(\cdot\mid i)$. Whenever $\pi_{w_{i'}} \neq \pi_{w_i}$ we have
$V^{\beta}(\cdot,i')\neq V^{\pi_{w_i}}(\cdot)$, so the perturbing term is nonzero and, barring
exact cancellation in expectation, $A^{\beta}((s,i),a) \neq A^{\pi_{w_i}}(s,a)$, giving
$\varepsilon_{\mathrm{sub}}(i) > 0$. The $\sigma$-weighting of the perturbing term
in~\eqref{eq:beta-bellman} makes the deviation grow with $\sigma$ and with the inter-version
value gap.

\emph{Entry into the improvement bound.} Let $\hat{L}(\pi_\theta)$ be the surrogate optimized
with the single-version advantage $A^{\pi_{w_i}}$, and $L^{\beta}(\pi_\theta)$ the surrogate in
the augmented bound~\eqref{eq:aug-bound} with the true mixed advantage $A^{\beta}$; both are
evaluated on the same sampling distribution $d_\beta$. Their difference is
\begin{equation*}
L^{\beta}(\pi_\theta) - \hat{L}(\pi_\theta)
= \frac{1}{1-\gamma}\,\mathbb{E}_{(s,i)\sim d_\beta,\, a\sim\pi_{w_i}}
\Big[ \tfrac{\pi_\theta(a\mid s)}{\pi_{w_i}(a\mid s)}
\big( A^{\beta}((s,i),a) - A^{\pi_{w_i}}(s,a) \big) \Big].
\end{equation*}
The total-variation penalty in~\eqref{eq:aug-bound} contains no advantage and is unchanged by
the substitution, so it cancels in $L^{\beta}-\hat{L}$; only the surrogate is affected.
Bounding the ratio by $M_\rho$ and applying the definition~\eqref{eq:adv-bias},
\begin{equation*}
\big| L^{\beta}(\pi_\theta) - \hat{L}(\pi_\theta) \big|
\;\le\; \frac{M_\rho}{1-\gamma}\, \mathbb{E}_{i}\big[\varepsilon_{\mathrm{sub}}(i)\big].
\end{equation*}
Writing $\hat{L} = L^{\beta} - (L^{\beta}-\hat{L})$ in the augmented bound~\eqref{eq:aug-bound}
and lower-bounding $L^{\beta}-\hat{L} \ge -|L^{\beta}-\hat{L}|$ moves this gap to the
right-hand side as an additive penalty, yielding~\eqref{eq:mono-improve-bias}. At $\sigma=0$ the
penalty vanishes and the standard GRPO objective realizes the bound's surrogate exactly.
\end{proof}

\section{A Control-Variate View of Baseline Choice in DORA vs.\ Partial Rollout}\label{sec:baseline-view}
Complementary to the policy-improvement analysis of Appendix~\ref{sec:policy-improvement},
which examined how single-version generation affects the monotonic-improvement guarantee,
this appendix turns to its effect on the \emph{variance of the policy-gradient estimator}.
DORA and partial-rollout methods~\citep{team2025kimi, wu2025llamarl, du2025ulorl, fu2025areal}
differ along a second, subtler algorithmic dimension: \emph{the choice of constant baseline
used in advantage estimation for group-relative policy optimization}. We examine this
dimension through the classical control-variate framework for policy
gradients~\citep{weaver2001optimal, greensmith2004variance}. Our goal is again modest: to
provide a clean lens through which two baseline choices can be located and compared, and to
describe the structure of the variance gap between them within a stylized abstraction.

\paragraph{Scope and abstraction.}
The analysis below is a \emph{stylized comparison of two constant baselines under a common
sampling abstraction}, not a faithful model of either system's sampling process. We reuse the
single-version notation of Appendix~\ref{sec:policy-improvement}: $\pi_\theta$ is the target
policy and $\{\pi_{w_k}\}$ the lagged versions with schedule $\{\alpha_k\}$. Concretely, we
adopt the following \emph{single-version sampling abstraction}---the $\sigma=0$ regime of
Appendix~\ref{sec:policy-improvement}: each sample is a triple $(k, s, a)$ with
$k \sim \alpha$, $s \sim \rho_0$, and $a \sim \pi_{w_k}(\cdot\mid s)$, so that the response is
generated end-to-end under a single policy version $\pi_{w_k}$. This abstraction is faithful
to DORA and serves as a stylized stand-in for version-mixing methods, whose stitched
trajectories it does not model in full. Both estimators studied below are defined with respect
to this abstraction; they share the same sampling distribution and the same importance-weighted
score $S_k(s,a) = (\pi_\theta/\pi_{w_k})\,\nabla_\theta \log \pi_\theta(a\mid s)$, and differ
only in the constant baseline subtracted from the reward.

This abstraction is faithful to DORA's actual sampling process, since DORA generates each trajectory end-to-end under a single policy version. It is \emph{not} faithful to partial-rollout methods, whose trajectories are stitched from segments produced under multiple policy versions within the staleness window.\footnote{For partial-rollout systems that apply token-level importance correction, the stitched and single-version sampling distributions agree in expectation, which provides a partial justification for the abstraction; their variance behavior nonetheless differs in general.} We make no claim that the analysis below captures partial rollout's true sampling variance. Instead, we use the version-specific baseline $\mu_k$ as the formal counterpart of DORA's design (which can name the version of any trajectory) and the schedule-averaged baseline $\bar{\mu} = \sum_k p_k \mu_k$ as the natural representative of the partial-rollout family within this abstraction (which by design does not single out any one version of a trajectory). The resulting comparison isolates the dimension of \emph{baseline choice} from orthogonal mechanisms used by concrete partial-rollout systems---decoupled objectives, importance corrections across stitched segments, gradient masking, and so on---which lie outside our scope. Following GRPO~\citep{shao2024deepseekmath} and the asynchronous RL systems built on it~\citep{fu2025areal, team2025kimi, wu2025llamarl}, we restrict attention to constant (state-independent) baselines; learned state-dependent baselines (as in PPO) are outside our scope. All proofs are deferred to Appendix~\ref{sec:baseline-proofs}.

\subsection{Setup and Notation}
\label{subsec:baseline-setup}

We adopt the sequence-level contextual bandit formulation~\citep{shao2024deepseekmath, ahmadian2024back}: states $s \sim \rho_0$ are prompts drawn from a static distribution, actions $a \sim \pi(\cdot|s)$ are full responses, and rewards $R(s, a) \in [0, R_{\max}]$ are verifiable correctness signals. Let $\pi_\theta$ denote the target policy and $\{\pi_{w_k}\}_{k=1}^K$ the lagged policies with schedule $\{p_k\}$. Define
\begin{equation}
\mu_k = \mathbb{E}_{s \sim \rho_0,\, a \sim \pi_{w_k}}[R(s, a)],
\qquad
\bar{\mu} = \sum_k p_k \mu_k,
\label{eq:mu-defs}
\end{equation}
and the importance-weighted score
\begin{equation}
S_k(s, a) = \frac{\pi_\theta(a|s)}{\pi_{w_k}(a|s)} \nabla_\theta \log \pi_\theta(a|s),
\label{eq:score}
\end{equation}
where $\nabla_\theta$ denotes the gradient with respect to the policy parameters $\theta$. Note that $\bar{\mu}$ is a deterministic constant, whereas $\mu_k$ is a random variable depending on the sampled version index $k$. We define three local second-order quantities, all conditional on a fixed version $k$ with $(s, a) \sim \rho_0 \times \pi_{w_k}$:
\begin{equation}
B_k = \mathbb{E}\!\left[\|S_k\|^2 \,\middle|\, k\right],
\quad
C_k = \mathbb{E}\!\left[(R - \mu_k)\|S_k\|^2 \,\middle|\, k\right],
\quad
V_k = \mathbb{E}\!\left[(R - \mu_k)^2 \|S_k\|^2 \,\middle|\, k\right].
\label{eq:moments}
\end{equation}
Since $\mu_k = \mathbb{E}[R \mid k]$ and $B_k = \mathbb{E}[\|S_k\|^2 \mid k]$, the quantity $C_k$ admits the equivalent interpretation
\begin{equation}
C_k = \mathrm{Cov}_{(s,a)\mid k}\!\left(R,\,\|S_k\|^2\right),
\label{eq:Ck-as-cov}
\end{equation}
i.e., the within-version covariance between reward and squared score. Throughout, $\mathrm{Var}(\cdot)$ denotes $\mathbb{E}[\|\cdot - \mathbb{E}\cdot\|^2] = \mathrm{tr}(\mathrm{Cov}(\cdot))$ for vector-valued random variables, and $\mathrm{Var}_p(\cdot)$, $\mathrm{Cov}_p(\cdot, \cdot)$ denote variance and covariance under the schedule distribution $p$ over versions.

The two estimators we compare are
\begin{equation}
\hat{g}_{\mu_k} \;=\; (R - \mu_k)\, S_k,
\qquad
\hat{g}_{\bar{\mu}} \;=\; (R - \bar{\mu})\, S_k,
\label{eq:estimators}
\end{equation}
both with $k \sim p$ and $(s, a) \sim \rho_0 \times \pi_{w_k}$ under the single-version sampling abstraction. Within this abstraction, $\hat{g}_{\mu_k}$ corresponds to DORA's design (which can identify the version of every trajectory), and $\hat{g}_{\bar{\mu}}$ corresponds to the partial-rollout family within the abstraction (which by design does not single out one version per trajectory).\footnote{We re-emphasize that $\hat{g}_{\bar{\mu}}$ is a stylized stand-in: concrete partial-rollout systems include additional mechanisms---decoupled objectives, importance corrections across stitched segments, gradient masking---that are not represented by $\hat{g}_{\bar{\mu}}$. The variance comparison below reflects only the baseline-choice dimension.} We analyze single-sample variance throughout; batch-averaged variance scales as $1/G$ where $G$ is the batch size, so the comparative results extend directly to the batched case. Under standard importance-sampling regularity (Appendix~\ref{subsec:regularity}), both estimators are unbiased for $\nabla_\theta J(\pi_\theta)$.

\subsection{Optimal Constant Baseline}
\label{subsec:optimal-baseline}

We first characterize the variance-minimizing constant baseline, which serves as the reference point against which both $\mu_k$ and $\bar{\mu}$ are approximations.

\begin{lemma}[Optimal Constant Control-Variate Baseline]
\label{lem:optimal-baseline}
For any fixed version $k$ with $B_k > 0$ and $\mathbb{E}[R^2 \|S_k\|^2 \mid k] < \infty$, the constant baseline $b \in \mathbb{R}$ minimizing $\mathrm{Var}((R - b)S_k \mid k)$ is unique and given by
\begin{equation}
b^\star_k \;=\; \frac{\mathbb{E}[R\,\|S_k\|^2 \mid k]}{\mathbb{E}[\|S_k\|^2 \mid k]} \;=\; \mu_k + \frac{C_k}{B_k}.
\label{eq:optimal-baseline}
\end{equation}
\end{lemma}

Equation~\eqref{eq:optimal-baseline} expresses $b^\star_k$ as the sum of two additive terms: the version-mean reward $\mu_k$, and a correction term $C_k / B_k$ capturing the within-version reward--gradient coupling. The second term requires per-sample gradient norm estimates and is generally impractical to compute in sequence generation settings---a difficulty noted explicitly by~\citet{hao2025on}, who similarly invoke simplifying assumptions to obtain a tractable form of the optimal baseline. Both baselines we consider drop this correction term:
\begin{itemize}
\item $\mu_k$ retains version-level reward information; within our abstraction, this is the baseline available to a method that can identify the version of every trajectory. For a given $k$, $\mu_k = b^\star_k$ when $C_k/B_k = 0$.
\item $\bar{\mu} = \sum_k p_k \mu_k$ aggregates across versions; within our abstraction, this is the baseline available to a method that does not. The condition $\bar\mu = b^\star_k$ for all $k$ requires both $C_k/B_k = 0$ for all $k$ and $\mu_k$ constant in $k$.
\end{itemize}
The next subsection quantifies the variance gap between these two choices.

\subsection{Exact Variance Decomposition}
\label{subsec:variance-gap}

The variance behavior of the two baselines is characterized by the following exact identity.

\begin{theorem}[Variance Gap Decomposition]
\label{thm:variance-gap}
Under the single-version sampling abstraction and the regularity conditions of Appendix~\ref{subsec:regularity}, the variance gap between the two estimators decomposes as
\begin{equation}
\Delta\mathrm{Var}
\;:=\;
\mathrm{Var}(\hat{g}_{\bar{\mu}}) - \mathrm{Var}(\hat{g}_{\mu_k})
\;=\;
\underbrace{\mathbb{E}_{k \sim p}\!\left[B_k\,(\mu_k - \bar{\mu})^2\right]}_{\text{drift term}}
\;+\;
\underbrace{2\,\mathrm{Cov}_p(\mu_k,\, C_k)}_{\text{coupling term}}.
\label{eq:variance-gap}
\end{equation}
\end{theorem}

\paragraph{Reading the decomposition.}
The decomposition isolates two distinct sources of variance difference. The drift term is non-negative by construction; the coupling term has indeterminate sign, so $\Delta\mathrm{Var}$ itself can in principle have either sign.

\begin{itemize}
\item The \textbf{drift term} $\mathbb{E}_{k\sim p}[B_k(\mu_k - \bar{\mu})^2]$ is the cross-version pass-rate variance weighted by gradient energy. It quantifies the cost, within the abstraction, of replacing the version-specific mean $\mu_k$ with a single cross-version aggregate $\bar{\mu}$: it grows whenever pass rates differ across versions, and vanishes when they coincide.
\item The \textbf{coupling term} $2\,\mathrm{Cov}_p(\mu_k, C_k)$ captures whether higher-pass-rate versions also tend to exhibit larger reward--gradient covariance (positive sign, amplifying the gap) or the opposite (negative sign, partially offsetting it). Its sign and magnitude depend on the joint distribution of $(\mu_k, C_k)$ across versions.
\end{itemize}

\subsection{Structural Observations}
\label{subsec:when-gap-matters}

Theorem~\ref{thm:variance-gap} reduces $\Delta\mathrm{Var}$ to two terms with distinct structural roles. We record a Cauchy--Schwarz bound on the coupling term and note what it does and does not imply.

\paragraph{Cauchy--Schwarz upper bound on the coupling term.}
Applying Cauchy--Schwarz under the schedule distribution $p$,
\begin{equation}
\bigl|\,2\,\mathrm{Cov}_p(\mu_k, C_k)\,\bigr|
\;\leq\;
2\,\sqrt{\mathrm{Var}_p(\mu_k)\,\mathrm{Var}_p(C_k)}.
\label{eq:cs-bound}
\end{equation}
Under the bounded-reward and bounded-score conditions of Appendix~\ref{subsec:regularity}, $|C_k| \leq 2 R_{\max} M$ for every $k$, so $\mathrm{Var}_p(C_k) < \infty$ and the inequality is well-defined. The bound expresses that the coupling term is constrained by the cross-version variability of $\mu_k$ and of the within-version reward--gradient covariance $C_k$; it does not by itself imply any direction for $\Delta\mathrm{Var}$.

\paragraph{What this analysis does and does not establish.}
We do not claim that $\hat{g}_{\mu_k}$ has uniformly smaller variance than $\hat{g}_{\bar{\mu}}$, even within the abstraction. Theorem~\ref{thm:variance-gap} is explicit that $\Delta\mathrm{Var}$ can have either sign, and the magnitudes of both terms depend on the training regime in ways we do not attempt to characterize. What the analysis does establish, within the single-version sampling abstraction, is: (i) the gap admits an exact decomposition into a non-negative drift term and a sign-indeterminate coupling term, and (ii) the coupling term is structurally bounded by~\eqref{eq:cs-bound}. We do not extrapolate to claims about the variance of full partial-rollout systems, which involve additional mechanisms outside the abstraction.

\subsection{Summary}
\label{subsec:baseline-summary}

This appendix has located two baseline choices within the control-variate framework, under a common single-version sampling abstraction. Both are tractable approximations to the optimal constant baseline $b^\star_k = \mu_k + C_k / B_k$ that drop the impractical coupling term $C_k / B_k$. The version-specific baseline $\mu_k$---available to methods, like DORA, that can identify the version of every trajectory---retains version-level reward information; the schedule-averaged baseline $\bar{\mu}$---used here as the representative within the abstraction of methods that do not---aggregates across versions. Theorem~\ref{thm:variance-gap} gives an exact decomposition of the variance gap into a non-negative drift term and a sign-indeterminate coupling term. The framework's value is descriptive: it offers a structural lens on the algorithmic dimension of single-policy-per-trajectory generation, complementing the system-level benefit (zero-re-prefill migration).

\section{Proofs for the Control-Variate Analysis}
\label{sec:baseline-proofs}
This appendix provides full proofs for the results in Appendix~\ref{sec:baseline-view}. We collect the regularity assumptions and unbiasedness identities in Appendix~\ref{subsec:regularity}, prove Lemma~\ref{lem:optimal-baseline} in Appendix~\ref{subsec:proof-optimal}, and prove Theorem~\ref{thm:variance-gap} in Appendix~\ref{subsec:proof-variance-gap}.

\subsection{Regularity Assumptions and Unbiasedness}
\label{subsec:regularity}

The analysis in Appendix~\ref{sec:baseline-view} relies on the following standard regularity conditions:
\begin{itemize}
\item \textbf{Bounded reward:} $R(s,a) \in [0, R_{\max}]$ almost surely.
\item \textbf{Bounded score:} For all versions $k$, $\|S_k(s, a)\|^2 \leq M$ almost surely under $(s, a) \sim \rho_0 \times \pi_{w_k}$. In practice, this is enforced by the importance-ratio clipping mechanism of PPO and GRPO~\citep{schulman2017proximal, shao2024deepseekmath}, which prevents the score from diverging. Combined with bounded reward, this implies $B_k$, $C_k$, $V_k$ are all finite, with $|C_k| \leq 2 R_{\max} M$ and $V_k \leq R_{\max}^2 M$ for every $k$.
\item \textbf{Non-degenerate gradient energy:} For all $k$, $B_k = \mathbb{E}[\|S_k\|^2 \mid k] \geq B_{\min} > 0$. This ensures that the target policy maintains a nontrivial gradient signal and prevents division-by-zero singularities in the derivation of the optimal baseline.
\item \textbf{Importance-sampling regularity:} The standard support condition $\pi_\theta \ll \pi_{w_k}$ holds for all $k$. By direct change of measure,
\begin{equation}
\mathbb{E}\!\left[S_k(s, a) \,\middle|\, k\right]
\;=\;
\mathbb{E}_{s \sim \rho_0}\!\left[ \int \pi_\theta(a|s) \nabla_\theta \log \pi_\theta(a|s)\, \mathrm{d}a \right]
\;=\;
\mathbb{E}_{s \sim \rho_0}[\nabla_\theta 1]
\;=\; 0,
\label{eq:score-zero}
\end{equation}
using the identity $\int \pi_\theta \nabla_\theta \log \pi_\theta\, \mathrm{d}a = \nabla_\theta \int \pi_\theta\, \mathrm{d}a = \nabla_\theta 1 = 0$. Furthermore, by the same change of measure,
\begin{equation}
\mathbb{E}\!\left[R(s, a)\, S_k(s, a) \,\middle|\, k\right]
\;=\;
\mathbb{E}_{(s,a) \sim \rho_0 \times \pi_\theta}\!\left[R(s, a)\, \nabla_\theta \log \pi_\theta(a|s)\right]
\;=\; \nabla_\theta J(\pi_\theta),
\label{eq:reward-score}
\end{equation}
which is independent of $k$.
\end{itemize}

\paragraph{Conditional mean of a generic baseline-shifted estimator.}
Conditional on a fixed version $k$, any constant $b$ is deterministic. Combining~\eqref{eq:score-zero} and~\eqref{eq:reward-score} yields, for every constant $b \in \mathbb{R}$,
\begin{equation}
\mathbb{E}[(R - b)\,S_k \mid k]
\;=\;
\mathbb{E}[R\,S_k \mid k] \;-\; b\,\mathbb{E}[S_k \mid k]
\;=\;
\nabla_\theta J(\pi_\theta) \;-\; b \cdot 0
\;=\;
\nabla_\theta J(\pi_\theta).
\label{eq:cond-mean}
\end{equation}
Specializing $b = \mu_k$ and $b = \bar{\mu}$ gives $\mathbb{E}[\hat{g}_{\mu_k} \mid k] = \mathbb{E}[\hat{g}_{\bar{\mu}} \mid k] = \nabla_\theta J(\pi_\theta)$. Since the conditional mean is $k$-independent, taking expectation over $k \sim p$ preserves it: both estimators are unbiased for $\nabla_\theta J(\pi_\theta)$.

\subsection{Proof of Lemma~\ref{lem:optimal-baseline}}
\label{subsec:proof-optimal}

We work conditional on a fixed version $k$ throughout this proof; all expectations are over $(s, a) \sim \rho_0 \times \pi_{w_k}$. The proof requires only the conditions stated in Lemma~\ref{lem:optimal-baseline}---namely $B_k > 0$ and $\mathbb{E}[R^2 \|S_k\|^2 \mid k] < \infty$---together with the importance-sampling regularity of Appendix~\ref{subsec:regularity} (used to establish~\eqref{eq:cond-mean}).

By definition of the trace covariance,
\begin{equation}
\mathrm{Var}((R - b)\, S_k \mid k)
\;=\;
\mathbb{E}\!\left[\,\|(R - b) S_k\|^2 \,\middle|\, k\right]
\;-\; \|\mathbb{E}[(R - b)\, S_k \mid k]\|^2.
\label{eq:trace-cov}
\end{equation}
Since $R - b$ is a scalar, $\|(R - b)\,S_k\|^2 = (R - b)^2 \|S_k\|^2$. By~\eqref{eq:cond-mean}, $\mathbb{E}[(R - b)\,S_k \mid k] = \nabla_\theta J(\pi_\theta)$ for every constant $b$, so the second term in~\eqref{eq:trace-cov} is independent of $b$. Minimizing $\mathrm{Var}((R - b)\,S_k \mid k)$ over $b \in \mathbb{R}$ is therefore equivalent to minimizing the scalar function
\begin{equation*}
f(b)
\;:=\;
\mathbb{E}\!\left[(R - b)^2 \|S_k\|^2 \,\middle|\, k\right]
\;=\;
\mathbb{E}[R^2 \|S_k\|^2 \mid k]
\;-\; 2b\, \mathbb{E}[R\, \|S_k\|^2 \mid k]
\;+\; b^2\, B_k.
\end{equation*}
This is a strictly convex quadratic in $b$ (since $f''(b) = 2 B_k > 0$), so the unique minimizer is given by the first-order condition $f'(b) = 0$:
\begin{equation*}
-2\, \mathbb{E}[R \|S_k\|^2 \mid k] \;+\; 2b\, B_k \;=\; 0
\quad\Longrightarrow\quad
b^\star_k \;=\; \frac{\mathbb{E}[R\, \|S_k\|^2 \mid k]}{B_k}.
\end{equation*}
To connect this to $\mu_k$, add and subtract $\mu_k$ in the numerator:
\begin{equation*}
b^\star_k
\;=\;
\frac{\mathbb{E}[((R - \mu_k) + \mu_k)\, \|S_k\|^2 \mid k]}{B_k}
\;=\;
\frac{\mu_k\, B_k + C_k}{B_k}
\;=\;
\mu_k + \frac{C_k}{B_k},
\end{equation*}
by the definitions of $B_k$ and $C_k$ in~\eqref{eq:moments}. 

\subsection{Proof of Theorem~\ref{thm:variance-gap}}
\label{subsec:proof-variance-gap}

\paragraph{Step 1: Reduction to conditional variances via the law of total variance.}
For each estimator $\hat{g} \in \{\hat{g}_{\mu_k}, \hat{g}_{\bar{\mu}}\}$, the matrix law of total covariance states
\[
\mathrm{Cov}(\hat g) \;=\; \mathbb{E}_{k\sim p}[\mathrm{Cov}(\hat g \mid k)] \;+\; \mathrm{Cov}_{k\sim p}(\mathbb{E}[\hat g \mid k]).
\]
Taking the trace of both sides and using linearity of trace yields the analogous identity for the trace-covariance scalar $\mathrm{Var}(\cdot)$ defined in Section~\ref{subsec:baseline-setup}:
\begin{equation}
\mathrm{Var}(\hat{g})
\;=\;
\mathbb{E}_{k \sim p}\!\left[\mathrm{Var}(\hat{g} \mid k)\right]
\;+\;
\mathrm{Var}_{k \sim p}\!\left[\mathbb{E}(\hat{g} \mid k)\right].
\label{eq:total-var}
\end{equation}
By Appendix~\ref{subsec:regularity}, $\mathbb{E}[\hat{g} \mid k] = \nabla_\theta J(\pi_\theta)$ for both estimators, which is $k$-independent. The second term in~\eqref{eq:total-var} therefore vanishes, yielding $\mathrm{Var}(\hat{g}) = \mathbb{E}_{k \sim p}[\mathrm{Var}(\hat{g} \mid k)]$.

\paragraph{Step 2: Conditional variance for a generic constant baseline.}
Fix $k$ and consider the estimator $\hat{g}_b = (R - b)\,S_k$ with $b$ a constant. Combining~\eqref{eq:trace-cov} with $\mathbb{E}[(R - b)\,S_k \mid k] = \nabla_\theta J(\pi_\theta)$ from~\eqref{eq:cond-mean},
\begin{equation}
\mathrm{Var}(\hat{g}_b \mid k)
\;=\;
\mathbb{E}\!\left[(R - b)^2 \|S_k\|^2 \,\middle|\, k\right]
\;-\; \|\nabla_\theta J(\pi_\theta)\|^2,
\label{eq:cond-var-generic}
\end{equation}
where the second term is independent of $b$.

\paragraph{Step 3: Algebraic decomposition of the second moment.}
For $b = \bar{\mu}$, complete the square around $\mu_k$:
\begin{equation*}
(R - \bar{\mu})^2
\;=\;
(R - \mu_k)^2
\;+\; (\mu_k - \bar{\mu})^2
\;+\; 2(R - \mu_k)(\mu_k - \bar{\mu}).
\end{equation*}
Multiplying by $\|S_k\|^2$ and taking conditional expectation, with $(\mu_k - \bar{\mu})$ a deterministic constant given $k$,
\begin{align}
\mathbb{E}[(R - \bar{\mu})^2 \|S_k\|^2 \mid k]
&= \underbrace{\mathbb{E}[(R - \mu_k)^2 \|S_k\|^2 \mid k]}_{= V_k}
\;+\; (\mu_k - \bar{\mu})^2\, \underbrace{\mathbb{E}[\|S_k\|^2 \mid k]}_{= B_k} \notag \\
&\quad +\; 2(\mu_k - \bar{\mu})\, \underbrace{\mathbb{E}[(R - \mu_k)\|S_k\|^2 \mid k]}_{= C_k} \notag \\
&= V_k + B_k(\mu_k - \bar{\mu})^2 + 2(\mu_k - \bar{\mu})\,C_k.
\label{eq:second-moment-PR}
\end{align}
For $b = \mu_k$, no completion is needed:
\begin{equation}
\mathbb{E}[(R - \mu_k)^2 \|S_k\|^2 \mid k] \;=\; V_k.
\label{eq:second-moment-DORA}
\end{equation}

\paragraph{Step 4: Aggregating over $k$ and subtracting.}
Combining Steps~1--3,
\begin{align*}
\mathrm{Var}(\hat{g}_{\bar{\mu}})
&= \mathbb{E}_{k\sim p}\!\left[V_k + B_k(\mu_k - \bar{\mu})^2 + 2(\mu_k - \bar{\mu})\,C_k\right]
\;-\; \|\nabla_\theta J(\pi_\theta)\|^2, \\
\mathrm{Var}(\hat{g}_{\mu_k})
&= \mathbb{E}_{k\sim p}[V_k] \;-\; \|\nabla_\theta J(\pi_\theta)\|^2.
\end{align*}
Subtracting, the $\mathbb{E}_p[V_k]$ terms and the $\|\nabla_\theta J(\pi_\theta)\|^2$ terms cancel exactly:
\begin{equation}
\Delta\mathrm{Var}
\;=\;
\mathbb{E}_{k\sim p}\!\left[B_k(\mu_k - \bar{\mu})^2\right]
\;+\;
2\,\mathbb{E}_{k\sim p}\!\left[(\mu_k - \bar{\mu})\,C_k\right].
\label{eq:gap-pre-cov}
\end{equation}

\paragraph{Step 5: Identifying the cross term as a covariance.}
Since $\mathbb{E}_{k\sim p}[\mu_k] = \bar{\mu}$, we have $\mathbb{E}_p[\mu_k - \bar{\mu}] = 0$. Therefore,
\begin{equation*}
\mathrm{Cov}_p(\mu_k, C_k)
\;=\;
\mathbb{E}_p[(\mu_k - \bar{\mu})(C_k - \mathbb{E}_p[C_k])]
\;=\;
\mathbb{E}_p[(\mu_k - \bar{\mu})\,C_k]
\;-\;
\underbrace{\mathbb{E}_p[\mu_k - \bar{\mu}]}_{= 0} \cdot \mathbb{E}_p[C_k]
\;=\;
\mathbb{E}_p[(\mu_k - \bar{\mu})\,C_k].
\end{equation*}
Substituting into~\eqref{eq:gap-pre-cov} yields the claimed identity:
\begin{equation*}
\Delta\mathrm{Var}
\;=\;
\mathbb{E}_{k\sim p}\!\left[B_k(\mu_k - \bar{\mu})^2\right]
\;+\;
2\,\mathrm{Cov}_p(\mu_k, C_k). 
\end{equation*}

\section{Social impacts}\label{sec:social-impact}
By detailing our methodology and experimental results in production environments, we aim to advance both research and industrial practices in the field of large language model (LLM) training. Through improving training efficiency, DORA has the potential to meaningfully reduce the environmental footprint associated with training LLMs, which typically demand hundreds or even thousands of accelerators. Furthermore, our methodology has been validated on large-scale, non-CUDA mid-range accelerators, broadening its applicability beyond conventional hardware ecosystems. This not only benefits large organizations seeking hardware flexibility but also contributes to the democratization of LLM training by making it more accessible on less advanced accelerators.



\end{document}